\documentclass[iicol]{sn-jnl}

\usepackage{graphicx}%
\usepackage{multirow}%
\usepackage{amsmath,amssymb,amsfonts}%
\usepackage{amsthm}%
\usepackage{mathrsfs}%
\usepackage[title]{appendix}%
\usepackage{xcolor}%
\usepackage{textcomp}%
\usepackage{manyfoot}%
\usepackage{booktabs}%
\usepackage{algorithm}%
\usepackage{algorithmicx}%
\usepackage{algpseudocode}%
\usepackage{listings}%

\usepackage[capitalize]{cleveref}

\usepackage{booktabs}
\usepackage{colortbl}
\usepackage{placeins}
\definecolor{mygray}{gray}{.9}

\newcommand\blue{\textcolor{blue}}

\newcommand\red{\textcolor{red}}

\theoremstyle{thmstyleone}%
\newtheorem{theorem}{Theorem}
\newtheorem{proposition}[theorem]{Proposition}%

\theoremstyle{thmstyletwo}%

\theoremstyle{thmstylethree}%
\newtheorem{definition}{Definition}%

\begin{document}

\title[Article Title]{Saving for the future: Enhancing generalization via partial logic regularization}


\author[1,2]{\fnm{Zhaorui} \sur{Tan}}\email{Zhaorui.Tan21@student.xjtlu.edu.com}
\equalcont{These authors contributed equally to this work.}

\author[1,2]{\fnm{Yijie} \sur{Hu}}\email{Yijie.Hu20@student.xjtlu.edu.com}
\equalcont{These authors contributed equally to this work.}

\author*[1]{\fnm{Xi} \sur{Yang}}\email{Xi.Yang01@xjtlu.edu.cn}

\author[1]{\fnm{Qiufeng} \sur{Wang}}\email{Qiufeng.Wang@xjtlu.edu.cn}

\author[2]{\fnm{Anh} \sur{Nguyen}}\email{Anh.Nguyen@liverpool.ac.uk}

\author*[3]{\fnm{Kaizhu} \sur{Huang}}\email{kaizhu.huang@dukekunshan.edu.cn}


\affil*[1]{\orgname{Xi'an Jiaotong-Liverpool University}, \orgaddress{\city{Suzhou}, \country{China}}}

\affil[2]{\orgname{University of Liverpool}, \orgaddress{\city{Liverpool}, \country{United Kingdom}}}

\affil[3]{\orgname{Duke Kunshan University }, \orgaddress{\city{Suzhou}, \country{China}}}


\abstract{
Generalization remains a significant challenge in visual classification tasks, particularly in handling unknown classes in real-world applications. Existing research focuses on the class discovery paradigm, which tends to favor known classes, and the incremental learning paradigm, which suffers from catastrophic forgetting. Recent approaches such as the L-Reg technique employ logic-based regularization to enhance generalization but are bound by the necessity of fully defined logical formulas, limiting flexibility for unknown classes.
This paper introduces PL-Reg, a novel partial-logic regularization term that allows models to reserve space for undefined logic formulas, improving adaptability to unknown classes. 
Specifically, we formally demonstrate that tasks involving unknown classes can be effectively explained using partial logic. We also prove that methods based on partial logic lead to improved generalization.
We validate PL-Reg through extensive experiments on Generalized Category Discovery,  Multi-Domain Generalized Category Discovery, and long-tailed Class Incremental Learning tasks, demonstrating consistent performance improvements. Our results highlight the effectiveness of partial logic in tackling challenges related to unknown classes.
}




\maketitle

\section{Introduction}\label{Introduction}

\begin{figure*}[t]
\centering
\includegraphics[width=\linewidth]{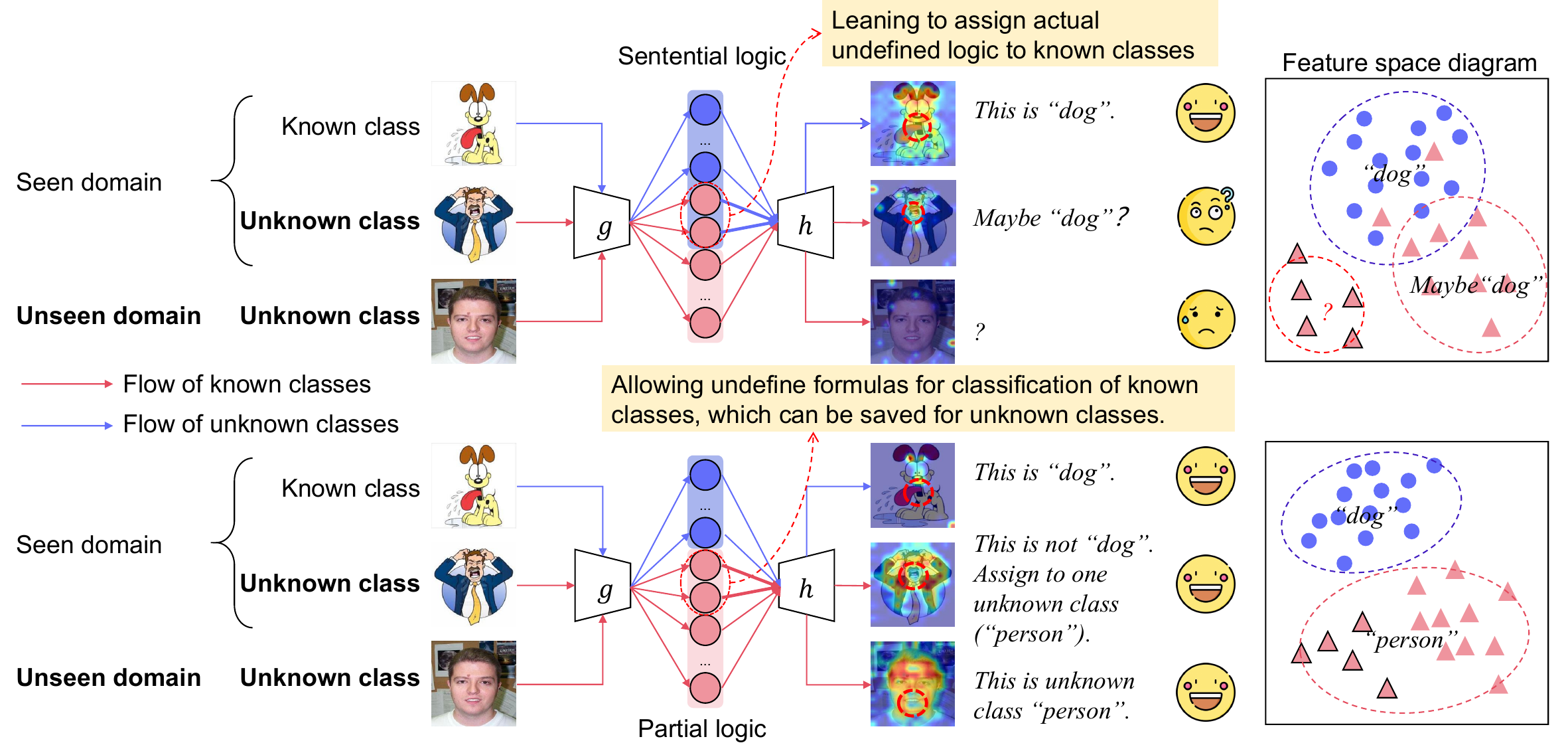}
\caption{
Diagrams of sentential logic- and partial logic-based visual classification, respectively.
}\label{fig:banner}
\end{figure*}

Generalization presents a persistent challenge in applying machine learning and deep learning models, particularly in visual classification tasks~\cite{gulrajani2020search,kukleva2021generalized,liu2021isometric,vaze2022generalized,cha2022miro,abbe2023generalization,tan2024rethinking,tan2024interpret}. As models are increasingly deployed in real-world scenarios, their ability to generalize to unknown classes becomes a critical factor in their success. Current research efforts in this area can be broadly categorized into two major branches: class discovering and incremental learning. 

The class-discovering paradigm enables models to identify unknown classes in unlabeled data without explicit supervision. A  fundamental task within this domain is Novel Class Discovery (NCD)~\cite{han2021autonovel}, where all unlabeled samples belong to unknown classes. This task is extended in Generalized Category Discovery (GCD)~\cite{vaze2022generalized}, encompassing a mixture of known and unknown classes within unlabeled samples. 
More recently, mDG+GCD~\cite{tan2024interpret} further complicates the setting by incorporating multi-domain generalization, allowing unknown classes to appear across unseen domains.
Alternatively, the incremental learning paradigm is designed to progressively adapt models to emerging classes. Class Incremental Learning (CIL)~\cite{rebuffi2017icarl,masana2022class,zhou2024class} simulates scenarios where novel classes are introduced sequentially, requiring models to integrate new knowledge while preserving prior information. However, traditional CIL assumes balanced data distributions, a condition that is rarely satisfied in practice. To address this limitation, ordered and shuffled long-tailed CIL~\cite{tao2020few,liu2022long} introduce class imbalances across sessions, posing greater challenges for generalization.

Despite advancements in class discovering and incremental learning, challenges persist in handling unknown classes across various scenarios. Class-discovering methods often exhibit bias toward known classes, significantly degrading the performance of unknown ones. Meanwhile, incremental learning methods struggle with catastrophic forgetting, losing previously learned knowledge while adapting to novel classes.
L-Reg~\cite{tan2024interpret} is a noteworthy attempt to mitigate challenges by introducing logic-based regularization to enhance interpretability and improve generalization in GCD and mDG+GCD.
Nevertheless, its reliance on fully defined logical formulas, i.e., \textit{sentential logic},  either imposes meanings on undefined terms or eliminates them. This reliance constrains the model's capability of unknown classes. 
As shown in \cref{fig:banner} top, methods based on sentential logic such as L-Reg~\cite{tan2024interpret}  often require full use of extracted features, causing semantic overlap among classes. These overlaps limit features to specific known classes, hindering the model's ability to adapt to novel classes, particularly in CIL, where logical formulas must evolve, potentially limiting the model's generalization.
Therefore, introducing partial-logic systems is essential to handle evolving class definitions and improve model flexibility and generalization. 

Drawing inspirations from the logic-based analysis introduced in L-Reg~\cite{tan2024interpret},
this paper unveils that the key to tackling this challenge is to reserve some room for the unknown classes during the training; such rationale can be well explained by \textit{partial logic} that allows the existence of currently undefined logic formulas. 
Through incorporating partial logic, models are able to retain undefined logic formulas of unknown classes and preserve them for the future stage where they will be defined. 
As one important contribution, we provide a theoretical analysis based on the logical analysis framework, showcasing that the problems of the aforementioned tasks can be formed under the scope of partial logic. Specifically, we prove that partial-logic based methods lead to improved generalization.
As shown in \cref{fig:banner} bottom, the partial logic-based method allows undefined formulas that would not be assigned to a certain known class at the current stage. Such reserved semantics hold great significance in that they may be strategically used to classify or cluster future unknown classes.

Building on the foundations of the partial logic,
we propose a novel partial-logic regularization term, named PL-Reg, apply it to a range of tasks, and validate our theoretical findings. 
To convey a comprehensive study, we conduct experiments with the following tasks that involve challenges about unknown classes:
\begin{itemize}
    \item \textbf{GCD}: In GCD tasks, models are trained on both labeled and unlabeled data from known and unknown classes. They classify unlabeled samples from known classes and cluster those for each unknown class. The total number of classes, including known and unknown, is given. 
    \item \textbf{mDG+GCD}: mDG+GCD tasks extend GCD by using labeled and unlabeled samples from several seen domains for training. The model classifies unlabeled samples from known classes and clusters those for each unknown class, but testing occurs on unseen domains.
    \item \textbf{Long-tailed CIL}: CIL tasks train models over several sessions. In each session, the model needs to learn to classify classes that are previously unknown, typically without access to past session samples. Long-tailed CIL tasks further complicate this as imbalanced data are involved.
\end{itemize}
As demonstrated from extensive results on the \textbf{GCD}, \textbf{mDG+GCD}, and \textbf{CIL} tasks shown, PL-Reg can lead to consistent improvements across various settings and datasets, confirming the superiority of taking partial logic into consideration. Our code is available at \url{https://github.com/zhaorui-tan/PL-Reg}.



\section{Related work}\label{Related work}

\subsection{Logical reasoning and visual tasks}

Current studies focus on length generalization~\cite{abbe2023generalization,ahuja2024provable,abbe2024provable,xiao2024theory} or symbolic reasoning~\cite{boix2023can,li2024neuro} in the logic-based scope, which are closely related to languages. 
For visual tasks, \cite{barbiero2022entropy} 
delves into the logical explanations in image classification by explicitly extracting logical relationships.

L-Reg~\cite{tan2024interpret} makes a first attempt to introduce the logical-based regularization through correlating the logical analysis framework with visual classification tasks. Moreover, L-Reg is able to improve generalization with interpretability. However, this paper unveils that L-Reg, based on sentential (binary) logic, may not be suitable for the scenarios where the unseen samples and unknown classes are incrementally given, thus possibly hindering the transferability of the model trained with L-Reg. 



\subsection{Generalization in visual classification tasks}
Following \cite{tan2024interpret}, the generalization settings for visual classification can be categorized as data-shift setting, target-shift setting, and all-shift setting, which is a combination of both. 
In this paper, we put our focus on those settings that involve target shifts, such as the generalized category discovery task for target-shift setting and combining multi-domain generalization and Generalized category discovery for all-shift setting.
Moreover, we revisit the Class Incremental Learning task as an extended target shift scenario.

\textbf{Multi-domain generalization.} 
Multi-domain generalization (mDG) frames classification tasks as classifying samples from unseen domains where the target shifts are significant.
Current approaches for mDG in image classification focus on learning invariant representation across domains~\cite{yuan2023domain}. 
Previous approaches like DANN~\cite{ganin2016domain} minimize feature divergences between source domains. CDANN~\cite{li2018deep}, CIDG~\cite{li2018domain}, and MDA~\cite{hu2020domain} consider conditions for learning conditionally invariant features. MIRO~\cite{cha2022miro} and GMDG~\cite{tan2024rethinking} take advantage of pre-trained models to improve generalization.

\textbf{Generalized category discovery.}
As a typical example of target-shift setting, 
generalized category discovery (GCD) is a task where a model trained on a partially labeled dataset needs to classify known classes covered by the labeled set and cluster unknown classes for unlabeled data. Pioneered by~\cite{vaze2022generalized} that addresses unlabeled samples with both known and unknown classes, PIM~\cite{chiaroni2023parametric}  integrates InfoMax into generalized category discovery, effectively handling imbalanced datasets and surpassing GCD on both short- and long-tailed datasets.

\textbf{Combining multi-domain generalization and generalized category discovery.}
A more complex task combing multi-domain generalization and generalized category discovery (mDG+GCD) is introduced by \cite{tan2024interpret}. In this task, the model is trained on partially labeled data from seen domains but tested on unseen domains for known class classification and unknown class clustering to validate its generalization ability.

\textbf{Class incremental learning.}
Class Incremental Learning (CIL) seeks to learn novel classes sequentially without access to previously encountered data, where the target shifts happen during incremental learning sessions. CIL aims to learn novel classes and retain the old knowledge in incremental sessions~\cite{rebuffi2017icarl, rolnick2019experience, zhu2021class, zhou2022model,xuan2024incremental}. One more realistic yet crucial task in CIL is to learn incoming novel classes with imbalanced distributions, namely long-tailed class incremental learning (LTCIL)~\cite{liu2022long}, where novel class data is heavily imbalanced compared to old class data, or the data distributions are imbalanced in each incremental stage. To tackle the imbalance issue, prevailing wisdom~\cite{liu2022long, kalla2024robust} first adopts various re-sampling methods to balance the data distribution in either input space or feature space, then learns the feature extractor and classifier via a two-stage framework. After training, the model is tested on all seen classes to verify its performance under the continual target shifts.


\section{Partial logical regularization}
\label{Method}


\textbf{Notations.}
Consider a set of paired data $(X, Y) \sim (\mathcal{X},   \mathcal{Y})$ represents inputs and corresponding labels.  Two distinct subsets are identified: $(X_k, Y_k) \sim (\mathcal{X}_k, \mathcal{Y}_k)$ represents the seen and known paired subsets, and $(X_u, Y_u) \sim (\mathcal{X}_u, \mathcal{Y}_u)$ represents the unseen paired subsets. Note that $X_u$ or $Y_u$ may be available independently, yet their pairing relationships remain undisclosed.
Here, $Y$  is presumed as a finite set. 
Let $D$ denote the finite set of possible domains, with $D_k$ and $D_u$ representing the seen and unseen domains, respectively. For classification tasks, an encoding function $g(x) \to Z \in \mathbb{R}^M$ is employed to map the input sample $x$ into a latent feature space $Z$, where each latent feature has $M$ dimensions. A predictor $h(Z) \to \hat{Y} \in \mathbb{R}^K$ is then used to map the latent features $Z$ into predictions $\hat{Y}$, where $K$ denotes the number of classes and the dimensionality of the predictions.
Typically,  $h$ is assumed to be a linear model.
Additionally, $P(\cdot)$ and $H(\cdot)$ represent probability and entropy, respectively.

In most cases discussed within this paper, it should be clarified again that $Y$ is confined to a finite set and the discussion is limited to single-label classification tasks, i.e., each input $x \in X$  is only given one ground truth label $y\in Y$. 

\subsection{Sentential logic for visual classification}
We begin by highlighting the connections between logical reasoning and visual classification tasks. 
As provided in \cite{tan2024interpret}, the logic defined on the given images and targets $(X,Y)$ can be formed as follows:
\begin{definition} 
\label{def:logic}
Following~\cite{andreka2017universal}, a logic $\mathcal{L}_{(X,Y)}$ is a five-tuple defined in the form:
\begin{equation}
\begin{split}
    \label{eq:logic}
    \mathcal{L}_{(X,Y)}= &
    \left\langle 
    F_{\mathcal{L}_{(X,Y)}},  
    M_{\mathcal{L}_{(X,Y)}}, \right . \\ &
    \left .
    \models_{\mathcal{L}_{(X_k,Y_k)}},
    mng_{\mathcal{L}_{(X,Y)}}, 
    \vdash_{\mathcal{L}_{(X,Y)}}
    \right\rangle, \nonumber
\end{split}
\end{equation}
where
\begin{itemize}
    \item $F_{\mathcal{L}_{(X,Y)}}$ is a set of all formulas of $\mathcal{L}$ that is constructed on images and labels ($X,Y$).  
    \item $M_{\mathcal{L}_{(X,Y)}}$ is a class called the class of all models (or possible worlds) of $\mathcal{L}_{(X,Y)}$. $M_{\mathcal{L}_{(X,Y)}}$ can be understood as different domains $D$ of $X$ or various training sessions in CIL. 
     \item $\models_{\mathcal{L}_{(X,Y)}}$ is the binary validity relation of $\mathcal{L}_{(X,Y)}$ : $\models_{\mathcal{L}_{(X,Y)}} \subseteq M_{\mathcal{L}_{(X,Y)}} \times F_{\mathcal{L}_{(X,Y)}}$. 
     The validation relationship refers to the relationship between the ground truth label of the image being given being true and its negations being false.  
     \item $mng_{\mathcal{L}_{(X,Y)}}: F_{\mathcal{L}_{(X,Y)}} \times M_{\mathcal{L}_{(X,Y)}} \longrightarrow \text { Sets }$
    where Sets is the class of all sets.      
    $mng_{\mathcal{L}_{(X,Y)}}$ is called the meaning function of $\mathcal{L}_{(X,Y)}$. Intuitively, $mng_{\mathcal{L}}$ extracts the meaning of the expressions that can be understood as the classifiers.
    \item $\vdash_{\mathcal{L}_{(X,Y)}}$ represents the provability relation of $\mathcal{L}_{(X,Y)}$, telling us which formulas are ``true" in which possible world and usually is definable from  $mng_{\mathcal{L}_{(X,Y)}}$, such as the estimation in the machine learning system. 
\end{itemize}
\end{definition}

As shown in \cref{def:logic}, the meaning of the formulations presented in $\mathcal{L}_{(X,Y)}$ are binary, indicating that $\mathcal{L}_{(X,Y)}$  is sentential or propositional logic $\mathcal{L}_{S}$:
\begin{definition}
\label{def:sentential_logic}
Connectives of $\mathcal{L}_{S}$ are $\wedge,\vee$ and $\neg$.
Given the set of formulas ${F}_{S}$, the class $M_{S}$ of models is
\begin{equation}
 M_{S}\overset{\text{def}}{=} \left \{ f: f\in^{S} \left \{0,1\right \}  \right \} .
\end{equation}
Here, $0,1$ are intended to denote the truth values ``false'' and ``true'', respectively.
\end{definition}

L-Reg~\cite{tan2024interpret} is derived from \cref{def:logic,def:sentential_logic},  forcing the model to maintain a good general logic $\mathcal{L}^*$ from $\mathcal{L}_{(X,Y)}$ through forming \textit{atomic formulas} on the semantic representations obtained by the encoders and the predictions of the classifiers form them. Specifically, $ h\circ g(x) \text{ belongs/not belongs to class } y \text{ in domain } d \to Ture/False, \text{ where } x, y, d \in X,Y,D$ which makes that $ \vdash_{(h \circ g(X_u), Y_u)} = \models_{(g(X_k), Y_k)}$ still holds. Given this equivalence, as shown in L-Reg,  $ \vdash_{(h \circ g(X_u), Y_u)}$ can be omitted; consequently, all $\vdash_{\cdot}$ are omitted in this paper. 


\textbf{Limitations.}
\cref{def:sentential_logic} presumes that all formulations in $\mathcal{L}_{(X,Y)}$ are defined, i.e., meaningful in the given worlds/specific situations.
However, Copenhagen's interpretation of quantum mechanics points out that certain statements are meaningless in specific situations, whereas these statements may acquire definitions in future scenarios, such as 
during class discovering or incremental learning processes. 
Designs such as L-Reg, which are based on \cref{def:sentential_logic}, might either
force the model to assign the meaning to inherently undefined formulations or erase those formulations in the given situations.
As presented in \cref{fig:banner}, this could necessitate the allocation of overlapping semantics across categories to a specific known class. Both would limit the model's capability to adapt to newly defined formulations in novel situations. 

\subsection{Partial logic for visual classification tasks}


Partial logics are designed to capture the idea that, in certain contexts, some statements may be devoid of meaning.
Specifically, if a statement 
$\varphi$ is meaningless in the given scenarios, then its negation 
$\neg\varphi$ is also meaningless.
Following~\cite{andreka2017universal}, the partial logic $\mathcal{L}_{P}$ is defined as follows:
\begin{definition}
\label{def:partial_logic}
The logic $\mathcal{L}_{P}$ is characterized by its connectives $\wedge,\vee,\neg, N$, where $N(\varphi)$ denotes a strong negation, expressing that $\varphi$ is either false or undefined (i.e., ``It is not the case that $\varphi$"). The set of formulas $F_{P}$ extends $F_k$ (formulas of sentential logic) by incorporating the unary connective $N$.
The class $M_P$ of models is
\begin{equation}
 M_P \overset{\text{def}}{=} \left \{ f: f\in^{P} \left \{0,1,2\right \} 
 \right \} .
\end{equation}
Here, $0,1,2$ are used to denote the truth values ``false'', ``true'', and ``undefined'', respectively.
If $2 \notin \left \{mng_{P}(\varphi,f), mng_{P}(\psi ,f)  \right \}$ holds, then $mng_{P}$ of $\varphi \wedge \psi, \varphi \vee \psi, \neg \varphi$ follows the standard definitions in classical logic  $\mathcal{L}_k$. 
Conversely, if $2$ embodies one of the meanings,
then $mng_{P}$ of  $\varphi \wedge \psi, \varphi \vee \psi, \neg \varphi$ is all $2$, indicating that their meaning is not defined in the given $\mathcal{L}_k$:
\begin{equation}
\begin{split}
    mng_{P} (N_{\varphi}, f) 
    \overset{\text{def}}{=} 
    \left \{\begin{matrix}  
    &0, &\text{ if } mng_{P}(\varphi, f) =1 
     \\
    &1, &\text{ otherwise. } \hfill
    \end{matrix}\right. \\
    f \models_{P} \text{ iff } mng_{P}(\varphi, f) = 1. 
\end{split}
\end{equation}
With this, the partial logic is defined: $\mathcal{L}_{P} = \left \langle 
F_{P}, M_{P}, mng_{P}, \models_{P}
\right \rangle $.
\end{definition}

With \cref{def:partial_logic}, it is able to redefine the $\mathcal{L}_{(X,Y)}$ in the form of partial logic:
\begin{equation}
    \mathcal{L}_{(X,Y)}\!\mid\! m
    \!\overset{\text{def}}{=}\! \! \left \{ 
     \mathcal{L}_{S(X,Y)},  \mathcal{L}_{P(X,Y)} \!\setminus \mathcal{L}_{S(X,Y)}
    \right \} \!\mid\! m,
    \nonumber
\end{equation}
where $m \in M_{(X,Y)}$ denotes the set of the seen scenarios of all possible scenarios combinations;  $\mathcal{L}_{P(X,Y)} \setminus \mathcal{L}_{S(X,Y)}$ denotes the fully undefined logic set for the given seen scenarios $m$. Importantly,  $ \mathcal{L}_{(X,Y)}$ is a \textit{formal partial logic} if $\mathcal{L}_{P(X,Y)} \!\setminus \mathcal{L}_{S(X,Y)} \neq \emptyset$. Note here $\mathcal{L}_{(X,Y)}= =\mathcal{L}_{P(X,Y)}$.

This section subsequently demonstrates that the challenges of different visual tasks can be formed under the scope of partial logic. Following this, we further show that using partial logic improves the generalization ability of various models. 

\subsubsection{Partial logic in GCD tasks}

Consider all possible situations for the GCD task, denoted as $M_{GCD(X, Y)} = \left \{ m_{(g(X), Y_k)}, m_{(g(X), Y)} \right \} $. Here, $m_{(g(X), Y_k)} $ and $m_{(g(X), Y)} $ represent situations where the logic is constructed based on the extracted formulations with known classes $Y_k $ and all classes $Y = \left \{ Y_k, Y_u \right \} $, respectively. 

\textbf{Scenario 1.} Intuitively, while the given situations are $m_{(g(X), Y_k)}$, the corresponding obtained logic $\mathcal{L}_{GCD(X,Y)}\!\mid\! m_{(g(X), Y_k)}$ consists of 
\begin{equation}
\begin{split}
    &\left \{\mathcal{L}_{S(X,Y)},  \mathcal{L}_{P(X,Y)} \!\setminus \mathcal{L}_{S(X,Y)} 
    \right \} \!\mid\!  m_{(g(X), Y_k)}, \\
    &   \text{where }
    \mathcal{L}_{S(X,Y)}  \!\mid\!  m_{(g(X), Y_k)} \overset{def}{=} \mathcal{L}_{(g(X),Y_k)}, \\
    & \;\;\;\;\;\;\;\;\;\;  
    \mathcal{L}_{P(X,Y)} \!\setminus \mathcal{L}_{S(X,Y)}  \!\mid\!  m_{(g(X), Y_k)}  \overset{def}{=}  \\
    & \;\;\;\;\;\;\;\;\;\;\;\;\;\;  
    \mathcal{L}_{GCD(X,Y)}\!\setminus \mathcal{L}_{(g(X),Y_k)}. 
\end{split}
\end{equation}
Notably, it is obvious that $ \mathcal{L}_{P(X,Y)} \!\setminus \mathcal{L}_{S(X,Y)}  \!\mid\!  m_{(g(X), Y_k)} \neq \emptyset$ for the GCD task since the logic for unseen labels  $\mathcal{L}_{S(X,Y_u)}\neq \emptyset$ needs to be constructed in the future and $ \mathcal{L}_{S(X,Y_u)} \in  \mathcal{L}_{P(X,Y)} \!\setminus \mathcal{L}_{S(X,Y)}  \!\mid\!  m_{(g(X), Y_k)} $. 
Therefore,  $\mathcal{L}_{GCD(X,Y)}\!\mid\! m_{(g(X), Y_k)}$ is a formal partial logic.
Under these situations, the model is required to extract defined formulations for known classes.
However, if all extracted formulations are used for constructing $\mathcal{L}_{(g(X),Y_k)}$, there will be no room left for constructing $\mathcal{L}_{S(X,Y_u)}$.

\textbf{Scenario 2.}
Given that the situations are $m_{(g(X), Y)}$,  the logic $\mathcal{L}_{GCD(X,Y)}\!\mid\! m_{(g(X), Y)}$ is formed as:
\begin{equation}
\begin{split}
    &\left \{\mathcal{L}_{S(X,Y)},  \mathcal{L}_{P(X,Y)} \!\setminus \mathcal{L}_{S(X,Y)}
    \right \} \!\mid\!  m_{(g(X), Y)}, \\
    & \text{where } \mathcal{L}_{S(X,Y)} \!\mid\!  m_{(g(X), Y)} \overset{def}{=} \\
    & \;\;\;\;\;\;\;\;\;\;\;\;\;\;  
    \left \{\mathcal{L}_{(g(X),Y_k)}, \mathcal{L}_{(g(X),Y_u)} \right \}. 
\end{split}
\end{equation}
In such case $\mathcal{L}_{P(X,Y)} \!\setminus \mathcal{L}_{S(X,Y)} \!\mid\!  m_{(g(X), Y)} \neq \emptyset$ may hold for many realistic visual task applications where images $X$ often contain ``noises'' corresponding to a certain label, which can be interpreted as undefined formulations for the given task. 
Combining two scenarios, $\mathcal{L}_{GCD(X,Y)}\!\mid\! m_{(g(X), Y_k)}$ is a formal partial logic. 

\subsubsection{Partial logic mDG+GCD tasks}

The mDG+GCD task extends the GCD task by considering multiple domains. Specifically, the set of all possible situations for the mDG+GCD task is given by:
$M_{mDG+GCD(X,Y)} = \left \{ m_{(g(X),Y_k, D_k)}, m_{(g(X),Y, D)} \right \}$ where $D_k$ and $D=\left \{D_k, D_u\right \}$ denote seen domains and all domains, respectively. Typically, $D_k\neq \emptyset$ and $D_u \neq \emptyset$  are both finite sets.
Similar to the GCD tasks, we present the logic formulation for the mDG+GCD task under two scenarios. 

\textbf{Scenario 1.} Given that situations are $ m_{(g(X),Y_k, D_k)}$, the corresponding obtained logic $\mathcal{L}_{mDG+GCD(X,Y)}\!\mid\!  m_{(g(X),Y_k, D_k)}$ is 
\begin{equation}
\begin{split}
    &\left \{\mathcal{L}_{S(X,Y)},  \mathcal{L}_{P(X,Y)} \!\!\setminus \mathcal{L}_{S(X,Y)} 
    \right \} \!\!\mid\!  m_{(g(X),Y_k, D_k)}, \\
    & \text{where }
    \mathcal{L}_{S(X,Y)}  \!\mid\! m_{(g(X),Y_k, D_k)} \overset{def}{=}\\
    & \;\;\;\;\;\;\;\;\;\;\;\;\;\;  
    \mathcal{L}_{(g(X),Y_k, D_k)}, \\
    & \;\;\;\;\;\;\;\;\;\;  
    \mathcal{L}_{P(X,Y)}  \! \!\setminus \!  \mathcal{L}_{S(X,Y)}  \! \!\mid\!  m_{(g(X), Y_k, D_k)} \!\!\! \overset{def}{=} \!\! \\
    & \;\;\;\;\;\;\;\;\;\;\;\;\;\;  
    \mathcal{L}_{mDG+GCD(X,Y)}\!\setminus \mathcal{L}_{(g(X),Y_k, D_k)}. 
\end{split}
\end{equation}
It is also obvious that $ \mathcal{L}_{P(X,Y)} \!\setminus \mathcal{L}_{S(X,Y)}  \!\mid\!  m_{(g(X), Y_k, D_k)}\neq \emptyset$ due to the same reasons presented for the GCD task. 

\textbf{Scenario 2.}
When the situations are given as
$m_{(g(X), Y, D)}$,  the logic $\mathcal{L}_{mDG+GCD(X,Y)}\!\mid\! m_{(g(X), Y, D)}$ is formed as
\begin{equation}
\begin{split}
    &\left \{\mathcal{L}_{S(X,Y)},  \mathcal{L}_{P(X,Y)} \!\setminus \mathcal{L}_{S(X,Y)}
    \right \} \!\mid\!  m_{(g(X), Y, D)}, \\
    & \text{where } \mathcal{L}_{S(X,Y)} \!\mid\!  m_{(g(X), Y, D)} \overset{def}{=} \\
    & \;\;\;\;\;\;\;\;\;\;\;\;\;\;  
    \left \{\mathcal{L}_{(g(X),Y_k, D_k)}, \mathcal{L}_{(g(X),Y_u, D)} \right \}. 
\end{split}
\end{equation}
Similar to the Scenario 2 in GCD tasks, $\mathcal{L}_{P(X,Y)} \!\setminus \mathcal{L}_{S(X,Y)} \!\mid\!  m_{(g(X), Y, D)}$ is not empty for the task of mDG+GCD.
The significant reason is that $X$ always contains domain-specific features which are not defined for category classification, as employed by mDG+GCD tasks. However, these features may be defined for other possible $Y$, such as different styles. 
Thus, 
$\mathcal{L}_{mDG+GCD(X,Y)}\!\mid\! m_{(g(X), Y, D)}$ is a formal partial logic.

\subsubsection{Partial logic in CIL tasks}
In contrast to the previously mentioned tasks, CIL requires a model to incrementally learn previously unknown classes across $Se$ times of sessions. Therefore, the set of all possible situations for CIL tasks is denoted as:
$M_{CIL(X,Y)} = \left \{ m_{(g(X),Y_k)}^i\right \}_{i=1}^{Se}$, 
where $m_{(g(X),Y_k)}^i$ represents the situation where the logic is constructed based on formulations with all known classes at the $i^{th}$ session. 

For each session $m_{(g(X),Y_k)}^i \in M_{CIL(X,Y)}$, 
the logic $\mathcal{L}_{CIL(X,Y)}\!\mid\! m_{(g(X),Y_k)}^i$ can be represented as
\begin{equation}
\begin{split}
    &\left \{\mathcal{L}_{S(X,Y)},  \mathcal{L}_{P(X,Y)} \!\setminus \mathcal{L}_{S(X,Y)}
    \right \} \!\mid\!  m_{(g(X),Y_k)}^i, \\
    & \text{where } \mathcal{L}_{S(X,Y)}  \!\mid\! m_{(g(X),Y_k)}^i \overset{def}{=} \\
    & \;\;\;\;\;\;\;\;\;\;\;\;\;\;  
    \left \{\mathcal{L}_{(g(X),Y_k)}^j \!\mid\! m_{(g(X),Y_k)}^j \right \}_{j=1}^{i}, \\
    & \;\;\;\;\;\;\;\;\;\;  
    \mathcal{L}_{P(X,Y)} \!\setminus \mathcal{L}_{S(X,Y)} \!\mid\!  m_{(g(X),Y_k)}^i \overset{def}{=}  \\ 
    & \;\;\;\;\;\;\;\;\;\;\;\;\;\; 
    \left \{\mathcal{L}_{(g(X),Y_k)}^j \!\mid\! m_{(g(X),Y_k)}^j \right \}_{j=i+1}^{Se}.
\end{split}
\end{equation}
When $i<Se$, it is obvious that $\left \{\mathcal{L}_{(g(X),Y_k)}^j \!\mid\! m_{(g(X),Y_k)}^j \right \}_{j=i+1}^{Se} \neq \emptyset$ and  $\mathcal{L}_{CIL(X,Y)}\!\mid\! m_{(g(X),Y_k)}^i$ is a formal partial logic. 
When $i=Se$, $\mathcal{L}_{CIL(X,Y)}\!\mid\! m_{(g(X),Y_k)}^i$ is a formal partial logic for most cases when  $X$ contains noises.
Moreover, the number of sessions denoted as $Se^*$, is typically finite in current CIL benchmarks, with $Se^* < < Se$, where $Se$  represents the total number of image classes across all possible images, reflecting a more realistic scenario. Therefore, $\mathcal{L}_{CIL(X,Y)}$ is a formal partial logic for most cases.

\begin{figure*}[!t]
\centering
\includegraphics[width=0.8\linewidth]{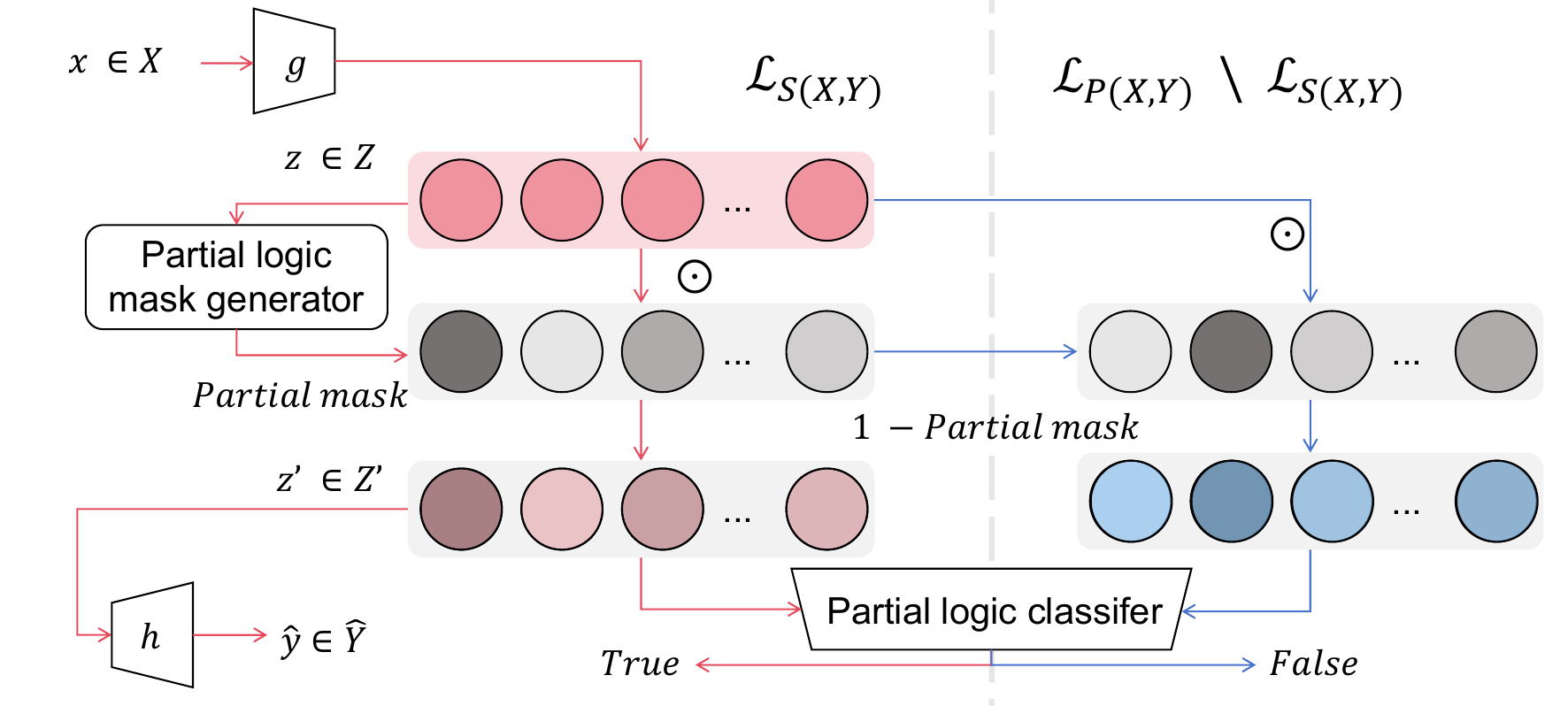}
\caption{Diagrams of applying PL-Reg. The diagram visualizes one example for simplification.
}\label{fig:app}
\end{figure*}

\subsubsection{Connections to improving generalization}
We now present that a model trained under the scope of partial logic improves the generalization of the models without considering partial logic for the aforementioned tasks.

For the visual classification task, we consider the feature extractor $g$ and the classifier $h$ as the whole model $f \overset{def}{=} mng(h\circ g)$. 
Though the previous tasks are formulated in different settings, improving the generalization of a model can be formed as minimizing the \textit{generalization loss}:
\begin{definition}[Generalization loss] 
    Let the target model $f^* \overset{def}{=} mng (h^*\circ g^*): f^*(X, Y): mng(X \to Y)$, can generalize across both seen and unseen sets $X, Y$. Denote that its trainable form is ${f}$, which is only trained on the seen sets. 
    The generalization loss on the $ \neg (X_k, Y_k) = (X, Y) \setminus (X_k, Y_k) $ is defined as:
    \begin{equation}
    \begin{split}
        &GL({f}, f^*, \neg (X_k, Y_k)  ) = \\
        &\;\;\;\;\mathbb{E}_{(x, y)\in \neg (X_k, Y_k) } || 
        f(x, y) - f^*(x, y)||_2.
    \end{split}
    \end{equation}
\end{definition}

We denote the model trained following \cref{def:sentential_logic} as $f_S \overset{def}{=} mng( h_S\circ g_S )$, and the other model training following \cref{def:partial_logic} as $f_P\overset{def}{=} (h_P \circ g_P)$. We present the following proposition:
\begin{proposition}
\label{prop:partial_better_generaliztion}
    Assume both models, $f_S$ and $f_P$ are well trained on $(X_k, Y_k)$ so that $||f_S(X_k, Y_k) - f^*||_2 = 0 $ and $ ||f_P(X_k, Y_k) - f^*||_2 = 0 $. 
    Under this assumption,
    $f_S = h_S\circ g_S$ gains more generalization loss than $f_P= h_P \circ g_P$:
    \begin{equation}
    \nonumber
    GL({f_S}, f^*, \neg (X_k, Y_k)  ) \geq 
    GL({f_P}, f^*, \neg (X_k, Y_k)),
    \end{equation}
    where equality is achieved iff the logic formed on $(X_k, Y_k)$ is not a formal partial logic.
\end{proposition}
\begin{proof}
    The proof is straightforward. For comprehensiveness and alignment with various tasks, $\neg (X_k, Y_k) $ is split into three possible subsets $\neg (X_k, Y_k) = \left \{  
    (X_k, Y_u), (X_u, Y_k),  (X_u, Y_u) 
    \right\}$, and the proof provides a discussion for each possible subset.

    \textbf{For $(X_k, Y_u)$.} For single-label classification, $X_k$ cannot be assigned to any unknown classes $Y_u$. Therefore, $f^*(X_u, Y_k) \in \{0\}$.
    As the well-trained $f_S$ and $f_P$ are assumed where $||f_S(X_k, Y_k) - f^*||_2 = 0 $ and $ ||f_P(X_k, Y_k) - f^*||_2 = 0 $, $||f_S(X_k, Y_u) - f^*||_2 = 0 $ and $ ||f_P(X_k, Y_u) - f^*||_2 = 0 $. It has:
    \begin{equation}
            \nonumber
             GL({f_S}, f^*, (X_k, Y_u)  ) = GL({f_P}, f^*, (X_k, Y_u)).
    \end{equation}
    This case requires no generalization and it can be aligned to standard single-label classification tasks. 
    
    \textbf{For $(X_u, Y_k)$.} In this case, it is possible that $X_u$ may or may not belong to the known classes. Therefore, $f^*(X_u, Y_k) \in \{0,1\}$. 
    Consequently, $||f_S(X_u, Y_k) \in \{0,1\} - f^*(X_u, Y_k)||_2 \geq ||f_P(X_u, Y_k) \in \{0,1, 2\}$. The equality can only be achieved when all $X_u$'s belong to the known classes. Hence:
    \begin{equation}
            \nonumber
            GL({f_S}, f^*, (X_u, Y_k)  ) \geq GL({f_P}, f^*, (X_u, Y_k)).
    \end{equation}
     Tasks like GCD and mGC+GCD both require such generalization, especially for their Situation~1. 

    \textbf{For $(X_u, Y_u)$.} In this case, $Y_u$ is unknown or undefined for the current situation and $X_u$ may belong to $Y_k$ but not $Y_u$ or belong to purely undefined $Y_u$.
    Thus, $f^*(X_u, Y_u) \in \{0, 2 \}$. $f_S$, however only allows defined relationships, thus  $f^*(X_u, Y_u) \in \{0\}$, while $f^*(X_u, Y_u) \in \{0, 2\}$.
    It is obvious that:
     \begin{equation}
            \nonumber
            GL({f_S}, f^*, (X_u, Y_u)  ) \geq GL({f_P}, f^*, (X_u, Y_u)),
    \end{equation}
    where the equality can only be achieved when every  $ (X_u, Y_u)  $ forms an empty set.
    Situation 1 of mGC+GCD and CIL both have $(X_u, Y_u)$ so that they require this kind of generalization. 
\end{proof}

\subsection{Deriving partial logic regularization}

While we have demonstrate the benefits of applying partial logic to visual tasks, it can be implemented in different forms. In the below, we introduce one intuitive yet practical  way of applying partial logic for validating \cref{prop:partial_better_generaliztion}.

As presented in \cref{def:partial_logic}, partial logic consists of sentential logic where some parts are defined, while others remain undefined. To enforce 
$ \mathcal{L}_{S(X,Y)} $ and $\mathcal{L}_{P(X,Y)} \!\setminus \mathcal{L}_{S(X,Y)} $, 
within the partial logic framework, we introduce a module to generate the partial logic mask $M$ for $Z$, and both of them have a shape of $[B, dim]$ where $B$ and $dim$ denote the batch size and the dimension of the extracted embeddings. 
In parallel, its negation $1-M$ is calculated for $\mathcal{L}_{P(X,Y)} \!\setminus \mathcal{L}_{S(X,Y)}$. 
Empirically, we use a linear layer with a sigmoid activation as the partial logic mask generator, aiming to avoid  additional parameters. 
With the generated mask, the extracted features are masked with $M$ and $1-M$ through multiplication: $Z\odot M, Z\odot(1_M)$, and they are concatenated and sent through the same liner classier $C$ to classify whether or not they have a defined meaning  with the following introduced loss:
\begin{equation}
\begin{split}
        {L}_{P1} = - \frac{1}{2B} \sum_{i=1}^{2B} &\left(  \right. Y_{P}^i \log(\hat{Y}_{P}^i) + \\
        & (1 - Y_{P}^i) \log(1 - \hat{Y}_{P}^i) \left. \right).
\end{split}
\end{equation}
Here $\hat{Y}_{P}^i = C(cat([Z\odot M, Z\odot(1_M)])) $, and $cat(\cdot)$ denotes concatenation at the first dimension;
$Y_{P}^i= cat([\textbf{1}, \textbf{0}])$, and $\textbf{1},\textbf{0}$ represents the one-hot label for the defined and undefined logic part, respectively. 

Additionally, to avoid the uniformly distributed mask values, which may cause the collapse of the logic construction, we further introduce the following loss:
\begin{equation}
    {L}_{P2} = - \frac{1}{B\cdot dim} \sum_{i=1}^{B} \sum_{j=1}^{dim} \left( m_{ij} \log(m_{ij}) \right),
\end{equation}
where $m=softmax(M)$ is the softmax output of $M$ at the last dimension; 
$m_{ij}$ is the $j^{th}$ element of the softmax output for the $i^{th}$ sample in the batch.
$\mathcal{L}_{P2}$ is applied to maintain the diversity of the mask for each dimension.

Finally, for the $\mathcal{L}_{S(X,Y)}$ part, the L-Reg is still requested to be applied for constructing the sentential logic:
\begin{equation}
\label{eq:loss}
\begin{split}
        {L}_{L-Reg} =  
        - \frac{1}{M}\sum_{i=1}^{M}\left[\sum_{j=1}^{K}a_{j,i}\log a_{j,i}\right] + \\
          \sum_{j=1}^{K}\left[\frac{1}{M}\sum_{i=1}^{M}a_{j,i} \log(\frac{1}{M}\sum_{i=1}^{M}a_{j,i})\right]
        ,
\end{split}
\end{equation}
where $a=softmax(\hat{Y}^{T} Z)$ is the softmax output of $\hat{Y}^{T} Z$ at the last dimension; 
$a_{j,i}$ denotes the value at the $i,j$ position of $softmax(\hat{Y}^{T}Z)$.

The overall partial logic regularization, termed PL-Reg, is a weighted combination of the aforementioned losses:
\begin{equation}
\begin{split}
    &{L}_{PL-Reg}\\
        = &\omega_{p1} {L}_{P1} + \omega_{p2}  {L}_{P2} + \omega_{L-Reg} {L}_{L-Reg},
\end{split}
\end{equation}
where $\omega_{p1}, \omega_{p2}$ and $\omega_{L-Reg}$ denote the weights. Empirically, we follow the weights of L-Reg in \cite{tan2024interpret} unless specificed explictly. 
Together with the other existing methods' losses denoted by $L_{main}$, the final loss for the training is:
\begin{equation}
    {L}_{final} = {L}_{PL-Reg} + L_{main}.
\end{equation}

\section{Experiments}
To validate the proposed PL-Reg, in this section, we conduct extensive experiments on the aforementioned tasks, including GCD (\cref{Experiments on GCD}), mDG+GCD (\cref{Experiments on mDG+GCD}), and CIL (\cref{Experiments on incremental learning}).
\subsection{Experiments on GCD}
\label{Experiments on GCD}

\begin{table*}[ht]
    \centering
    \caption{Statistics of GCD datasets.}
    \label{tab:stats}
    \begin{tabular}{l|cccccc}
        \hline
                & CUB  & Standford Cars & Herbarium19 & CIFAR10 & CIFAR100 & ImageNet-100 \\
        \hline
        Known classes & 100  & 98             & 341         & 5       & 80       & 50           \\
        Seen data & 1.5K & 2.0K           & 8.9K        & 12.5K   & 20K      & 31.9K        \\ \hline
        ALL classes & 200  & 196            & 683         & 10      & 100      & 100          \\
        Unseen data & 4.5K & 6.1K           & 25.4K       & 37.5K   & 30K      & 95.3K        \\
        \hline
    \end{tabular}
\end{table*}

\subsubsection{Experimental settings}
We follow the experimental details introduced in  L-Reg~\cite{tan2024interpret} for GCD tasks. 
Similarly, we validate our approach through training PIM additionally with PL-Reg.

\textbf{Competitors.}
Our main competitor is L-Reg~\cite{tan2024interpret} applied to PIM~\cite{chiaroni2023parametric} (PIM+L-Reg). In line with~\cite{tan2024interpret}, we compare our proposed method against several existing generalized category discovery methods, including GCD~\cite{vaze2022generalized} and PIM~\cite{chiaroni2023parametric}. Additionally, we consider traditional machine learning methods, such as k-means~\cite{macqueen1967classification}, as well as three novel category discovery approaches: RankStats+\cite{han2021autonovel}, UNO+\cite{fini2021unified}, and ORCA~\cite{cao2022openworld}. 
Several information maximization methods, including RIM~\cite{krause2010discriminative} and TIM~\cite{boudiaf2020information}, are also adapted for generalized category discovery and serve as competitors.

\textbf{Datasets.}
Consistent with previous works, we use six image datasets to evaluate the feasibility of our proposed PL-Reg in comparison to the other competitors. These datasets include three generic object recognition datasets: CIFAR-10~\cite{krizhevsky2009learning}, CIFAR-100~\cite{krizhevsky2009learning}, and ImageNet-100~\cite{deng2009imagenet}; two fine-grained datasets: CUB~\cite{wah2011caltech} and Stanford Cars~\cite{krause20133d}; and a long-tail dataset, Herbarium19~\cite{tan2019herbarium}. 

\textbf{Task protocols.} {Following the protocols of the aforementioned previous work, we divide the initial training set of each dataset into labeled and unlabeled subsets. Samples from half of the classes are designated as unlabeled, and their labels are not used for training. Specifically, half of the image samples from known classes are assigned to the labeled subset, while the remaining half are allocated to the unlabeled subset. Furthermore, the unlabeled subset includes all image samples from novel classes in the original dataset. Consequently, the unlabeled subset consists of instances from 
$K$ different classes. Detailed statistics of the datasets are provided in \cref{tab:stats}.
}

\textbf{Evaluation metric.}
{Following prior works \cite{vaze2022generalized,chiaroni2023parametric,tan2024interpret}, we adopt the  accuracy metric of all classes, known classes, and unknown classes for evaluation. 
}



\textbf{Training details.} 
Similar to PIM and PIM+L-Reg, we exploit latent features extracted by the feature encoder DINO (VIT-B/16)~\cite{caron2021emerging}, which is pre-trained on ImageNet~\cite{deng2009imagenet} through self-supervised learning.
Similar to what is implemented in L-Reg~\cite{tan2024interpret}, the implementation of PL-Reg is also applied to the latent features of the PIM backbone, where the partial logic masks and classifications are conducted. These masked latent features are subsequently utilized for the application of L-Reg.
The losses proposed in PIM are treated as 
$L_{main}$. The values of hyper-parameters used for training are presented in \cref{tab:GCD_hyper}. All other omitted hyper-parameters follow the setting of \cite{tan2024interpret}.

\begin{table}[t]
\caption{GCD experimental settings: Hyper-parameters. We follow the weights of L-Reg in \cite{tan2024interpret} as $\omega_{L-Reg}$, where $\lambda$ is the search weight of InfoMax losses presented in PIM~\cite{chiaroni2023parametric}.}\label{tab:GCD_hyper}
\centering
\begin{tabular*}{\linewidth}{@{\extracolsep\fill}lccc}
\hline
 & $\omega_{p1}$ & $\omega_{p2}$ &$\omega_{L-Reg}$ \\ \hline
CUB &  1. & 5e-1  & 1e-1\\
Stanford Cars & 5. & 5e-1 & 1e-3$\cdot\lambda$\\
Herbarium19 & 1.5e2 & 1e2 & 2e-1$\cdot\lambda$  \\
CIFAR100 &  1. & 5e-1 & 2.5e-4$\cdot\lambda$ \\
CIFAR10 & 1e3 & 5.& 1e-2$\cdot\lambda$\\
Imagenet-100 & 1. & 5e-1 & 1e-2$\cdot\lambda$\\
 \hline
\end{tabular*}%
\end{table}


\begin{table}[t]
\caption{\textbf{GCD results}: Averaged accuracy scores for all, known and unknown classes across all six datasets. The best results are highlighted in \blue{blue}.}\label{tab:GCD_summary}
\begin{tabular*}{\linewidth}{lccc}
\hline
 & All & Known & Unknown \\ \hline
k-means & 44.7 & 46.0 & 43.9 \\
RankStats+ (TPAMI-21) & 38.6 & 54.6 & 25.6 \\
UNO+ (ICCV-21) & 51.2 & 74.5 & 36.7 \\
ORCA (ICLR-22) & 46.3 & 51.3 & 41.2 \\
ORCA - ViTB16 & 56.7 & 65.6 & 49.9 \\
GCD (CVPR-22) & 60.4 & 71.8 & 52.9 \\
RIM (NeurIPS-10) & 62.0 & 72.5 & 55.4 \\
TIM (NeurIPS-20) & 62.7 & 72.6 & 56.4 \\ \hline
PIM (ICCV-23) & 67.4 & \blue{79.3} & 59.9 \\
\rowcolor{mygray}\; + L-Reg (NeurIPS-24) & 68.8 & 79.0 & 62.7 \\
\rowcolor{mygray}\; \textbf{+ PL-Reg} & \blue{\textbf{69.3}} & {78.8} & \blue{\textbf{63.7}} \\ 
\hline
\end{tabular*}%
\end{table}

\begin{table*}[t]
\caption{\textbf{GCD results}: Detailed accuracy scores for all, known and unknown classes across all six datasets. The best results are highlighted in \blue{blue}.}\label{tab:GCD_all}
\resizebox{\textwidth}{!}{%
\begin{tabular}{l|ccc|ccc|ccc}
\hline
\multicolumn{1}{c|}{} & \multicolumn{3}{c|}{\textbf{CUB}} & \multicolumn{3}{c|}{\textbf{Stanford Cars}} & \multicolumn{3}{c}{\textbf{Herbarium19}} \\
\multicolumn{1}{c|}{} & {All} & {Known} & {Unknown} & {All} & {Known} & {Unknown} & {All} & {Known} & {Unknown} \\ \hline
k-means & 34.3 & 38.9 & 32.1 & 12.8 & 10.6 & 13.8 & 12.9 & 12.9 & 12.8 \\
RankStats+ (TPAMI-21) & 33.3 & 51.6 & 24.2 & 28.3 & 61.8 & 12.1 & 27.9 & 55.8 & 12.8 \\
UNO+ (ICCV-21) & 35.1 & 49.0 & 28.1 & 35.5 & \blue{70.5} & 18.6 & 28.3 & 53.7 & 14.7 \\
ORCA (ICLR-22) & 27.5 & 20.1 & 31.1 & 15.9 & 17.1 & 15.3 & 22.9 & 25.9 & 21.3 \\
ORCA - ViTB16 & 38.0 & 45.6 & 31.8 & 33.8 & 52.5 & 25.1 & 25.0 & 30.6 & 19.8 \\
GCD (CVPR-22) & 51.3 & 56.6 & 48.7 & 39.0 & 57.6 & 29.9 & 35.4 & 51.0 & 27.0 \\
RIM (NeurIPS-10) & 52.3 & 51.8 & 52.5 & 38.9 & 57.3 & 30.1 & 40.1 & \blue{57.6} & 30.7 \\
TIM (NeurIPS-20) & 53.4 & 51.8 & 54.2 & 39.3 & 56.8 & 30.8 & 40.1 & 57.4 & 30.7 \\ \hline
PIM (ICCV-23) & 62.7 & 75.7 & 56.2 & 43.1 & 66.9 & 31.6 & 42.3 & 56.1 & 34.8 \\
\rowcolor{mygray}\; + L-Reg & 65.3 & 76.0 & 60.0 & 44.8 & 66.0 & 34.6 & \blue{43.7} & 55.8 & 37.2 \\
\rowcolor{mygray}\; \textbf{+ PL-Reg} & \blue{\textbf{67.1}} & \blue{\textbf{76.9}} & \blue{\textbf{62.3}} & \blue{\textbf{45.0}} & \textbf{65.6} & \blue{\textbf{35.1}} & \textbf{43.6} & \textbf{54.2} & \blue{\textbf{37.9}} \\
\hline
\multicolumn{1}{c|}{} & \multicolumn{3}{c|}{\textbf{CIFAR10}} & \multicolumn{3}{c|}{\textbf{CIFAR100}} & \multicolumn{3}{c}{\textbf{ImageNet-100}} \\
\multicolumn{1}{c|}{} & {All} & {Known} & {Unknown} & {All} & {Known} & {Unknown} & {All} & {Known} & {Unknown} \\ \hline
k-means & 83.6 & 85.7 & 82.5 & 52.0 & 52.2 & 50.8 & 72.7 & 75.5 & 71.3 \\
RankStats+ (TPAMI-21) & 46.8 & 19.2 & 60.5 & 58.2 & 77.6 & 19.3 & 37.1 & 61.6 & 24.8 \\
UNO+ (ICCV-21) & 68.6 & \blue{98.3} & 53.8 & 69.5 & 80.6 & 47.2 & 70.3 & 95.0 & 57.9 \\
ORCA (ICLR-22) & 88.9 & 88.2 & 89.2 & 55.1 & 65.5 & 34.4 & 67.6 & 90.9 & 56.0 \\
ORCA - ViTB16 & \blue{97.1} & 96.2 & 97.6 & 69.6 & 76.4 & 56.1 & 76.5 & 92.2 & 68.9 \\
GCD (CVPR-22) & 91.5 & 97.9 & 88.2 & 70.8 & 77.6 & 57.0 & 74.1 & 89.8 & 66.3 \\
RIM (NeurIPS-10) & 92.4 & 98.1 & 89.5 & 73.8 & 78.9 & 63.4 & 74.4 & 91.2 & 66.0 \\
TIM (NeurIPS-20) & 93.1 & 98.0 & 90.6 & 73.4 & 78.3 & 63.4 & 76.7 & 93.1 & 68.4 \\ \hline
PIM (ICCV-23) & 94.7 & 97.4 & 93.3 & 78.3 & 84.2 & 66.5 & 83.1 & \blue{95.3} & 77.0 \\
\rowcolor{mygray}\; + L-Reg & 94.8 & {97.6} & 93.4 & 80.8 & 84.6 & 73.2 & 83.4 & 94.0 & 78.0 \\
\rowcolor{mygray}\; \textbf{+ PL-Reg} & {\textbf{94.9}} & \textbf{97.5} & \blue{\textbf{93.6}} & \blue{\textbf{81.5}} & \blue{\textbf{84.9}} & \blue{\textbf{74.9}} & \blue{\textbf{83.5}} & \textbf{93.7} & \blue{\textbf{78.3}} \\
\hline
\end{tabular}%
}
\end{table*}

\subsubsection{Overall results}
The averaged results across all datasets are presented in \cref{tab:GCD_summary}.
Compared with L-Reg, using PL-Reg results in averaged $0.5\%$ improvements for all classes.
PIM+PL-Reg achieves the highest overall accuracy at $69.3\%$, making it the top performer in this comparison. The addition of PL-Reg provides a boost across all datasets, especially when handling unknown classes.

The accuracy for known classes is slightly lower at $78.8\%$, compared to PIM and PIM+L-Reg ($79.3\%$ and $79.0\%$, respectively). 
This indicates a small trade-off between known class performance and handling the unknown classes.
Those compromises are due to the fact that PL-Reg inherits limitations from L-Reg. Please refer to \cref{Discussion} for more discussions. 

The most significant improvement is seen in the unknown class accuracy, where PIM+PL-Reg achieves $63.7\%$, outperforming PIM ($59.9\%$) and PIM+L-Reg ($62.7\%$). This suggests that PL-Reg provides a substantial benefit for discovering and classifying novel categories, making PIM+PL-Reg the most well-rounded model in terms of accuracy across all classes. The consistent improvements of the unknown classes validate our theory in \cref{prop:partial_better_generaliztion}.

\subsubsection{Results across datasets}
\label{sec:GCD_detialed_analysis}
We conduct detailed analysis of each dataset in this section.

\textbf{Results of generic datasets: CIFAR10, CIFAR100, and Imagenet-100.}
As shown in \cref{tab:GCD_all}, the performance of the methods on generic datasets like CIFAR10, CIFAR100, and ImageNet-100 indicates that PIM with PL-Reg generalizes across common object recognition tasks. 
PIM+PL-Reg consistently outperforms other methods across all datasets, demonstrating significant improvements over PIM and PIM+L-Reg. Specifically, on CIFAR100, PIM+PL-Reg achieves a notable increase in the all classes accuracy ($81.5\%$) compared to PIM+L-Reg ($80.8\%$), as well as improvements in the unknown class category ($74.9\%$ compared to $73.2\%$). The CIFAR10 and ImageNet-100 results reveal a similar pattern, where PIM+PL-Reg outperforms PIM+L-Reg, particularly showing improvements in the unknown classes. 
The consistency of these improvements, suggests that PL-Reg enhances the model's ability to generalize across different classes, particularly the unknown classes.

\textbf{Results of fine-grained datasets: CUB and Stanford Cars.}
The fine-grained datasets, CUB and Stanford Cars, exhibit more significant differences between the methods as shown in \cref{tab:GCD_all}, especially in the case of PIM+PL-Reg. On CUB, PIM+PL-Reg achieves the highest overall accuracy ($67.1\%$), which is a substantial improvement over PIM ($62.7\%$) and PIM+L-Reg ($65.3\%$). For Stanford Cars, PIM+PL-Reg also outperforms the baseline methods, particularly in the unknown class accuracy ($35.1\%$ compared to $31.6\%$ in PIM). 
The improvement in the unknown class performance is more pronounced in Stanford Cars in comparison to PIM+L-Reg ($35.1\%$ for PIM+PL-Reg vs. $34.6\%$ for PIM+L-Reg). This suggests that PL-Reg offers better generalization when it comes to fine-grained category discovery compared to L-Reg, especially for unknown classes that are more challenging to identify due to their subtle differences.

\textbf{Results of long-tailed dataset: Herbarium19.}
The results in \cref{tab:GCD_all} on Herbarium19, a long-tailed dataset, also validate the advantages of PL-Reg in handling unknown classes. While the compromise in the known classes seems pronounced in this case (see discussions in \cref{Discussion}),  
PIM+PL-Reg shows notable improvements over PIM ($34.8\%$) for the unknown classes and achieves the second-best all classes results. 
The consistent improvements for unknown classes underline the importance of PL-Reg for addressing the challenges posed by long-tailed distribution in datasets. 

\subsection{Experiments on mDG+GCD}
\label{Experiments on mDG+GCD}


\subsubsection{Experimental settings}
We conduct mDG + GCD experiments on the same datasets  used in \cite{tan2024interpret}.

\textbf{Competitors.}
In our experiment, the GMDG serves as the baseline, and our main competitor is L-Reg~\cite{tan2024interpret} applied to GMDG~\cite{tan2024rethinking}. 
Results of the other methods (including ERM~\cite{gulrajani2020search}, PIM~\cite{chiaroni2023parametric}, and MIRO~\cite{cha2022miro}) are also adopted for comparison.

\textbf{Datasets.} 
{We leverage the datasets utilized in mDG tasks to construct the mDG+GCD datasets:
PACS~\cite{li2017deeper}, VLCS~\cite{fang2013unbiased}, OfficeHome~\cite{venkateswara2017deep}, TerraIncognita~\cite{beery2018recognition}, and DomainNet~\cite{peng2019moment}. Each dataset contains several domains. For example, \cref{fig:domain_vis} illustrates a few samples of all domains in the PACS dataset.  
}

\textbf{Task protocols.}
During training, only samples from seen domains are accessible, with half of the classes masked as unknown, where those samples of unknown classes are adopted as the unlabeled data. 
For instance, in the PCAS dataset comprising $7$ classes, classes labeled within the range $[0, 1, 2, 3]$ are retained, while classes in $[4, 5, 6]$ are masked.
It is noteworthy that data categorized as unknown classes in our setup are from unknown classes and used as unlabelled. 
However, we acknowledge that this prior is not explicitly known. To align with the GCD setting, we follow \cite{tan2024interpret} and operate under the assumption that the unlabeled set may potentially include samples from known classes.
As a result, we avoid constraining the model to classify unlabeled data solely as unknown classes, introducing a more challenging and realistic generalization scenario.



\textbf{Evaluation metric.}
Following the mDG approach, we also adopt leave-one-out cross-validation, where each domain in the dataset is tested as the unseen domain. 
We use the same metric from the GCD task for the mDG+GCD task. Similarly, the metrics include the accuracy for all, known, and unknown classes.
All evaluation details of PL-Reg are the same as those for validating L-Reg.  
Specifically, the model used for testing is selected during the training, where it achieves the best performance for known classes in the seen domains.

\textbf{Training details.}
For all experiments, partial logic masks and classifications are applied to the latent features of the GMDG backbone, which also serves as the basis for implementing L-Reg~\cite{tan2024interpret}. These masked latent features are subsequently utilized for the application of L-Reg.
The models are trained with the aforementioned labeled and unlabeled sets from the seen domains and tested on the samples from the unseen domain, with training settings consistent with L-Reg~\cite{tan2024interpret}.
The losses proposed in GMDG are treated as 
$L_{main}$. The values of hyper-parameters used for training are shown in \cref{tab:mDG+GCD_hyper}. All other omitted hyper-parameters are kept the same as those in \cite{tan2024interpret}.

\begin{figure}[t]
\centering
\includegraphics[width=\linewidth]{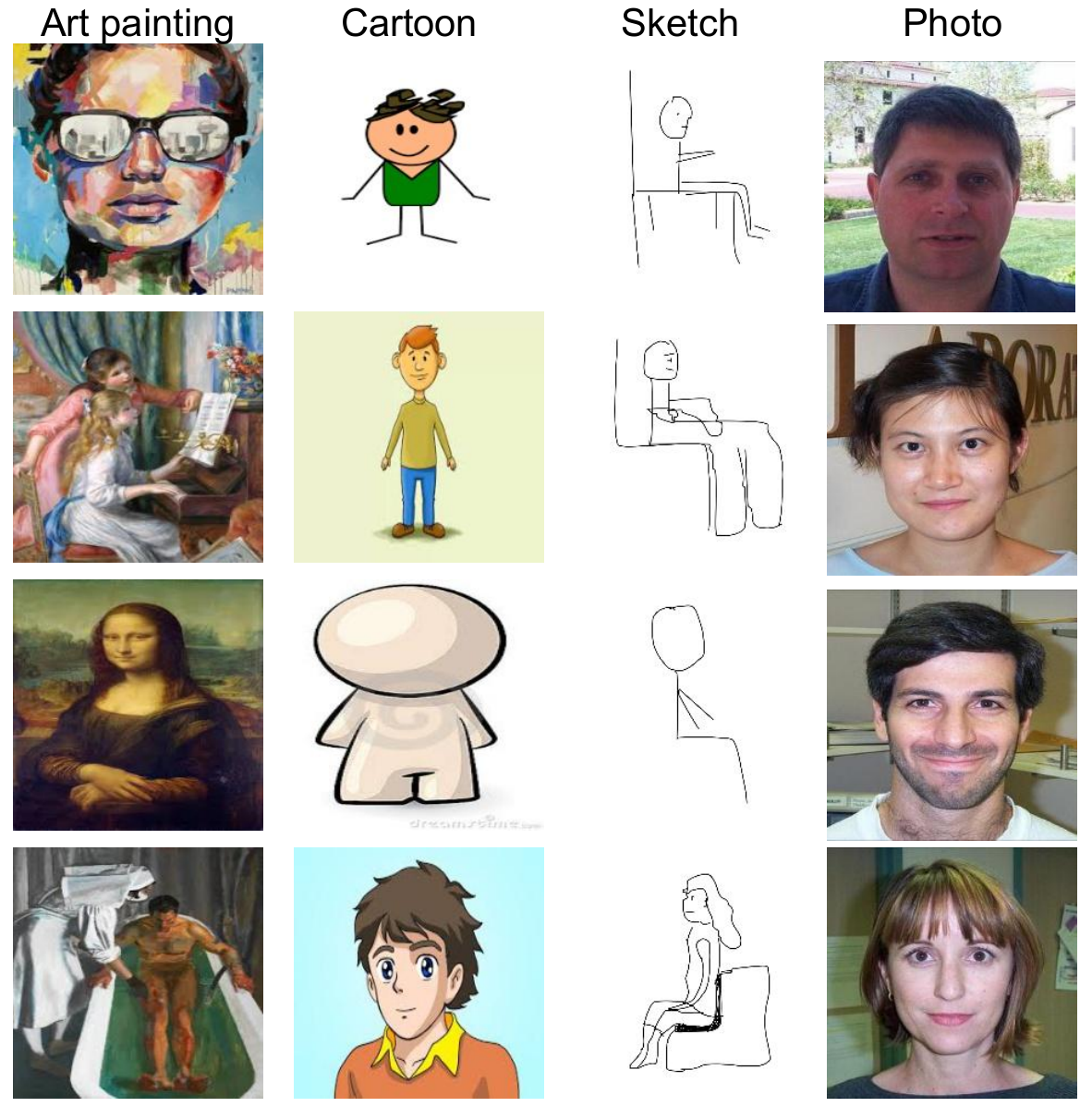}
\caption{
MDG+GCD visualization: Samples from all four domains in PACS datasets.
}\label{fig:domain_vis}
\end{figure}

\begin{table}[t]
\caption{MDG+GCD experimental settings: Hyper-parameters. The lr-mult hyper-parameter is the multiplication factor of the learning rate that was introduced previously~\cite{gulrajani2020search}.} \label{tab:mDG+GCD_hyper}
\centering
\begin{tabular*}{\linewidth}{@{\extracolsep\fill}lcccc}
\hline
 & $\omega_{p1}$ & $\omega_{p2}$ &$\omega_{L-Reg}$ & lr-mult\\ \hline
PACS &  5e-2 &  5e-2 &  1e-1 & 5e-1\\ 
HomeOffice  & 1e-1 & 5e-2 &  1e-1 & 1.\\
VLCS  &  1e2  & 5e-2 & 1e-1 & 5.\\ 
TerraIncognita & 7.5 & 5e-2 & 1e-1 & 4.5\\ 
DomainNet  & 1e-1 & 5e-2 &  1e-1  & 1. \\ 
 \hline
\end{tabular*}%
\end{table}


\begin{table*}[t]
\caption{MDG+GCD results: accuracy scores of each dataset. Improvements are highlighted in \red{red}.
The best results of each domain for each category are highlighted in \blue{blue}; Imp. denotes marginal improvements gained from L-Reg where improvements are highlighted in \red{red}.
}
\label{tab:avg_mdg_gcd_more_results}
\resizebox{\linewidth}{!}{%
\begin{tabular}{@{}ll|ccc|ccc|ccc|ccc|ccc@{}}
\hline
             &                      & \multicolumn{3}{c|}{\textbf{PACS}}                                                         & \multicolumn{3}{c|}{\textbf{HomeOffice}}                                                   & \multicolumn{3}{c|}{\textbf{VLCS}}                                                         & \multicolumn{3}{c|}{\textbf{TerraIncognita}}                                                & \multicolumn{3}{c}{\textbf{DomainNet}}                                                     \\ 
Method       & & {All}  & {Known}  & {Unknown} & {All}  & {Known}  & {Unknown} & {All}  & {Known}  & {Unknown} & {All}  & {Known}  & {Unknown}        & {All}  & {Known}  & {Unknown}         \\ \hline
ERM           &          & 57.26                        & 77.77                        & 22.33                        & 44.80                        & 74.67                        & 8.50                         & 61.51                        & 82.89                        & 34.88                        & 37.34                        & 20.46                        & 45.15                         & 22.56                        & 40.89                        & 6.85                         \\
PIM          &          & 56.35                        & 71.06                        & 27.43                        & 43.42                        & 72.44                        & 8.13                         & 63.19                        & 80.34                        & 40.24                        & 47.75                        & 35.31                        & 50.85                         & 24.03                        & 42.59                        & 7.86                         \\
MIRO         &             & 56.83                        & 85.62                        & 24.85                        & 48.28                        & 80.61                        & 9.03                         & 61.53                        & 82.72                        & 35.03                        & \blue{50.22}                        & \blue{39.92}                        & {49.45}                         & {31.49}                       & 55.44                        & 10.57                        \\\hline
GMDG         &   & 58.33                        & 91.46                        & 10.18                        & 48.85                        & \blue{81.41}                        & 9.22                         & 61.36                        & \blue{83.31}                        & 33.75                        &  40.02	& 32.38	 & 40.07 & {31.15} & 55.17                        & 10.18                        \\
\rowcolor{mygray}\; {+L-Reg}        &                      & 67.82                        & 91.86                        & 31.33                        & 51.96                        & 79.74                        & 18.15                        & 62.32                        & 82.77                        & 36.09                        &45.86& 39.77&	41.55                         & 31.75                        & 55.18                        & 11.30                        \\

\rowcolor{mygray} \; \textbf{+PL-Reg}& & \textbf{\blue{74.21}}& 	\textbf{\blue{92.07}}& 	\textbf{\blue{42.69}}& \textbf{\blue{54.73}}&	80.98&	\textbf{\blue{22.58}}& \textbf{\blue{65.10}}& 	80.13& 	\textbf{\blue{45.87}}& 48.93&	31.65&		\textbf{\blue{54.28}}& \textbf{\blue{32.54}}&	\textbf{\blue{55.53}}&	\textbf{\blue{12.37}}
\\ 
\rowcolor{mygray}\; Imp.& & \red{6.39} & \red{0.21}&	\red{11.36}& \red{2.78}&	\red{1.25}&	\red{4.44}& \red{2.78}& -2.64& 	\red{9.78}&  \red{3.08} & {-8.13} &	\red{12.73} & \red{0.79}&	\red{0.34}& \red{1.07}
\\
\hline
\end{tabular}%
}
\end{table*}
\begin{table}[t]
\caption{MDG+GCD results: Averaged accuracy scores for all, known and unknown classes across all five datasets. The best results are highlighted in \blue{blue}.
}
\label{tab:MDG+GCD_res}
\centering
\begin{tabular*}{\linewidth}{lccc}
\toprule
Method         & All                         & Known                         & Unknown                          \\ \hline
ERM             & 44.69                       & 59.33                       & 23.54                        \\
PIM (ICCV-23)            & 46.95                       & 60.35                       & 26.90                        \\
MIRO (ECCV-22)            & 49.67                       & 68.86                       & 25.79                        \\ \hline
GMDG (CVPR-24)          &47.94& 	68.75& 	20.68             \\
\rowcolor{mygray}\; {+L-Reg} (NeurIPS-24)  &51.94	&\blue{69.87}&	27.68                \\
\rowcolor{mygray}\; \textbf{+PL-Reg} &{\blue{55.10}}	&68.07	&{\blue{35.56}}  \\
\rowcolor{mygray}\; Imp. & \red{3.16} & -1.79 &	\red{7.87} \\ 
	
\bottomrule
\end{tabular*}%
\end{table}

\begin{table*}[!t]
\caption{MDG+GCD results: accuracy scores of each domain in PACS dataset. 
The best results of each domain for each category are highlighted in \blue{blue}; Imp. denotes marginal improvements gained from L-Reg where improvements are highlighted in \red{red}. $^*$ denotes the L-Reg reproduced by using the same $w_{L-Reg}$ that the PL-Reg used.
}
\label{tab:mdg_gcd_more_results1}
\resizebox{\linewidth}{!}{%
\begin{tabular}{l|ccc|ccc|ccc|ccc}
\hline
 & \multicolumn{3}{c|}{\textbf{art painting}} & \multicolumn{3}{c|}{\textbf{cartoon}}      & \multicolumn{3}{c|}{\textbf{photo}}        & \multicolumn{3}{c}{\textbf{sketch}}        \\ 
Method      & {All}  & {Known}  & {Unknown} & {All}  & {Known}  & {Unknown} & {All}  & {Known}  & {Unknown} & {All}  & {Known}  & {Unknown} \\  \hline
ERM & 47.77         & 90.00        & 0.00         & 56.08        & 83.49        & 20.47        & 59.13        & 47.35        & 68.85        & 66.06        & 90.23        & 0.00         \\
PIM & 46.80         & 55.17        & 37.32        & 50.37        & 89.15        & 0.00         & 62.05        & 49.50        & 72.40        & 66.19        & 90.40        & 0.00         \\
MIRO & 51.86         & 97.70        & 0.00         & 56.45        & \blue{99.91}        & 0.00         & 48.35        & 75.17        & 26.23        & 70.64        & 69.72        & {73.16}        \\
\rowcolor{mygray}\; \textbf{+ L-Reg}$^*$
&54.18& 98.28& 4.29&55.28& 86.89& 14.22&61.08& \blue{99.83}& 29.10&70.58 &96.05& 0.95 \\
\rowcolor{mygray}\; \textbf{+ PL-Reg} &54.97&98.16& 6.11&60.29& 97.74& 11.64&67.66& 99.67& 41.26&70.17& 68.85& \blue{73.75}\\
\hline
GMDG  & 51.92         & 97.82        & 0.00         & 54.80        & 96.98        & 0.00         & 56.14        & 74.83        & 40.71        & 70.45        & 96.22        & 0.00         \\
\rowcolor{mygray}\; + L-Reg & \blue{79.26}         & 98.05        & \blue{58.00}       & 68.18        & {99.25}        & 27.82        & 52.40        & 74.50        & 34.15        & \blue{71.47}       & 95.66        & {5.34}         \\
\rowcolor{mygray}\; \textbf{+ PL-Reg} &  {76.69}	&{\blue{98.28}}	&{52.28}	&{\blue{74.52}}	&{98.77}	&{\blue{43.01}}&	{\blue{74.78}}&	{{74.67}}&	{\blue{74.86}}	&{{70.87}}&	{\blue{96.57}} &{0.59}
\\
\rowcolor{mygray}\; Imp.& -2.56&	\red{0.23}&	-5.72	&\red{6.34}&	-0.47&	\red{15.20}&	\red{22.38}&	\red{0.17}&	\red{40.71}&	-0.60&	\red{0.91}& -4.75 \\
\hline
\end{tabular}%
}
\end{table*}

\begin{table*}[!t]
\caption{MDG+GCD results: accuracy scores of each domain in HomeOffice dataset.
The best results of each domain for each category are highlighted in \blue{blue}; Imp. denotes marginal improvements gained from L-Reg where improvements are highlighted in \red{red}.
}
\label{tab:mdg_gcd_more_results2}
\resizebox{\linewidth}{!}{%
\begin{tabular}{l|ccc|ccc|ccc|ccc}
\hline
    & \multicolumn{3}{c|}{\textbf{Art}}          & \multicolumn{3}{c|}{\textbf{Clipart}}      & \multicolumn{3}{c|}{\textbf{Product}}      & \multicolumn{3}{c}{\textbf{Real World}}    \\ 
Method      & {All}  & {Known}  & {Unknown} & {All}  & {Known}  & {Unknown} & {All}  & {Known}  & {Unknown}        & {All}  & {Known}  & {Unknown}         \\  \hline
ERM & 45.26        & 72.68        & 3.26         & 37.94        & 64.48        & 10.19        & 46.71        & 78.74        & 9.87         & 49.28        & 82.80        & 10.68        \\
PIM & 42.53        & 68.09        & 3.39         & 35.77        & 56.75        & 13.83        & 47.27        & 77.58        & 12.41        & 48.11        & 87.35        & 2.90         \\
MIRO & 50.57        & 79.57        & 6.13         & 39.55        & 67.23        & 10.60        & 51.35        & 86.16        & 11.32        & 51.66        & \blue{89.50}        & 8.09         \\
\hline
GMDG & 51.60        & \blue{81.96}        & 5.08         & 40.89        & \blue{69.30}        & 11.19        & 51.15        & \blue{87.53}        & 9.32         & 51.75        & 86.87        & 11.30        \\
\rowcolor{mygray}\; +L-Reg & 52.83        & 79.15        & 12.52        & 43.59        & 69.02        & 16.99        & 56.31        & 83.11        & 25.48        & 55.11        & 87.67        & 17.59        \\
\rowcolor{mygray}\; \textbf{+PL-Reg} & 	{\blue{54.79}} & {81.87}&		{\blue{13.30}} &	{\blue{44.22}}&	{68.57} &	{\blue{18.75}}&	{\blue{57.94}}&	{84.47}& {\blue{27.42}}&		{\blue{61.99}}&{89.01}	&{\blue{30.86}}\\
\rowcolor{mygray}\; Imp.& \red{1.96}	&\red{2.72}& \red{0.78}	&\red{0.63}&	-0.45&	\red{1.76}	&\red{1.63}&	\red{1.37}&	\red{1.94}&	\red{6.88}	&\red{1.34}&	\red{13.27}
\\
\hline
\end{tabular}%
}
\end{table*}
\begin{table*}[!t]
\caption{MDG+GCD results: accuracy scores of each domain in VLCS dataset.
The best results of each domain for each category are highlighted in \blue{blue}; Imp. denotes marginal improvements gained from L-Reg where improvements are highlighted in \red{red}.
}
\label{tab:mdg_gcd_more_results3}
\resizebox{\linewidth}{!}{%
\begin{tabular}{@{}l|ccc|ccc|ccc|ccc} 
\hline
 & \multicolumn{3}{c|}{\textbf{Caltech101}}   & \multicolumn{3}{c|}{\textbf{LabelMe}}      & \multicolumn{3}{c|}{\textbf{SUN09}}        & \multicolumn{3}{c}{\textbf{VOC2007}}       \\ 
Method & {All}  & {Known}  & {Unknown} & {All}  & {Known}  & {Unknown} & {All}  & {Known}  & {Unknown}        & {All}  & {Known}  & {Unknown}         \\  \hline
ERM & 82.07        & \blue{74.87}        & 85.85        & 50.54        & 92.01        & 4.85         & 62.07        & 95.15        & 11.38        & 51.35        & \blue{69.51}        & 37.45        \\
PIM & 80.39        & 72.05        & 84.77        & \blue{53.84}        & 91.74        & \blue{12.07}        & 62.22        & 94.21        & 13.21        & 56.31        & 63.36        & 50.92        \\
MIRO & 82.77        & 74.10        & 87.33        & 51.81        & 91.83        & 7.72         & 62.22        & \blue{95.59}        & 11.09        & 49.32        & 69.34        & 33.99        \\
\hline
GMDG & 82.51        & \blue{74.87}        & 86.52        & 49.93        & \blue{95.24}        & 0.00         & 59.86        & 93.96        & 7.62         & 53.13        & 69.17        & 40.85        \\
\rowcolor{mygray} \; +L-Reg & \blue{84.54}        & 74.62        & \blue{89.76}        & 49.98        & 92.01        & 3.66         & 61.39        & 95.03        & 9.84         & 53.39        & 69.43        & 41.11        \\
\rowcolor{mygray} \; \textbf{+PL-Reg}  &{79.95}	 &{73.59}	 &{83.29} &{51.20} &	{92.46} &	{5.74} &	{\blue{68.35}} &{88.55} &	{\blue{37.42}} &	{\blue{60.90}} &	{65.93} &	{\blue{57.06}}
\\
\rowcolor{mygray}\; Imp.&
-4.59&	-1.03&	-6.47&	\red{1.22}&	\red{0.45}&	\red{2.08}&	\red{6.97}&	-6.48 &  \red{27.58}&	\red{7.52}&	-3.50&	\red{15.95}
\\
\hline
\end{tabular}%
}
\end{table*}
\begin{table*}[t]
\caption{MDG+GCD results: accuracy scores of each domain in TerraIncognita dataset.
The best results of each domain for each category are highlighted in \blue{blue}; Imp. denotes marginal improvements gained from L-Reg where improvements are highlighted in \red{red}.
}
\label{tab:mdg_gcd_more_results4}
\resizebox{\linewidth}{!}{%
\begin{tabular}{@{}l|ccc|ccc|ccc|ccc}
\hline
 & \multicolumn{3}{c|}{\textbf{Local 100}} & \multicolumn{3}{c|}{\textbf{Local 38}}      & \multicolumn{3}{c|}{\textbf{Local 43}}        & \multicolumn{3}{c}{\textbf{Local 46}}        \\ 
Method & {All}  & {Known}  & {Unknown} & {All}  & {Known}  & {Unknown} & {All}  & {Known}  & {Unknown}        & {All}  & {Known}  & {Unknown}         \\  \hline
ERM & 46.51         & 1.25         & 57.07        & 39.88        & 28.22        & 44.91        & 29.41        & 24.65        & 40.25        & 33.59        & 27.70        & 38.36        \\
PIM & 50.20         & 28.97        & 55.15        & \blue{56.22}        & 19.71        & 71.99       & 46.69        & 47.94        & {43.86}        & 37.88        & \blue{44.64}        & 32.40        \\
MIRO & 52.23         & \blue{51.25}        & 52.46        & 55.54        & 14.73        & \blue{73.16}        & 48.93        & 62.89        & 17.13        & \blue{44.19}        & 30.79        & 55.06        \\
\hline
GMDG & 36.70 & 35.65 & 36.94 & 36.69 & 20.86 & 43.53 & 49.46 & 61.76 & 21.47 & 37.24 & 11.24 & \blue{58.33} \\
\rowcolor{mygray}\; +L-Reg & 51.89 & 38.16 & 55.09 & 41.30 & 22.05 & 49.61 & 50.47 & \blue{65.29} & 16.72 & 39.77 & 33.59 & 44.79 \\
\rowcolor{mygray}\; +\textbf{PL-Reg} &
\blue{54.63}	&31.06&	\blue{60.13}&	53.69&	12.01&	71.69	&\blue{52.24}&	55.82&	\blue{44.07}	&35.18	&27.70 & 41.25
\\
\rowcolor{mygray}\; Imp.&
\red{2.74} &	-7.10 &	\red{5.04} &	\red{12.39} &	-10.05 &	\red{22.08} &	\red{1.76} &	-9.47 &	\red{27.35} &	-4.59 &	-5.88 &	-3.54
\\

\hline
\end{tabular}%
}

\end{table*}

\begin{table*}[t]
\caption{MDG+GCD results: accuracy scores of each domain in DomainNet dataset.
The best results of each domain for each category are highlighted in \blue{blue}; Imp. denotes marginal improvements gained from L-Reg where improvements are highlighted in \red{red}.}
\label{tab:mdg_gcd_more_results5}
\centering
\resizebox{1.\linewidth}{!}{%
\begin{tabular}{@{}l|ccc|ccc|ccc}
\hline
 & \multicolumn{3}{c|}{\textbf{clipart}}      & \multicolumn{3}{c|}{\textbf{info}}         & \multicolumn{3}{c}{\textbf{painting}}   
\\ Method & {All}  & {Known}  & {Unknown} & {All}  & {Known}  & {Unknown} & {All}  & {Known}  & {Unknown}        \\ 
\hline
ERM & 31.04        & 58.32        & 7.15         & 17.94        & 34.71        & 6.85         & 30.59        & 51.82        & 9.34                \\
PIM & 32.01        & 57.38        & 9.80         & 18.80        & 33.56        & 9.03         & 22.22        & 36.62        & 7.80               \\
MIRO & 40.13        & 67.55        & 16.11        & 25.84        & {48.53}        & 10.84        & 37.89        & 62.45        & \blue{13.29}             \\
\hline
GMDG & 40.38        & \blue{70.69}        & 13.84        & 24.96        & 46.50        & 10.72        & 36.29        & 59.80        & 12.75               \\
\rowcolor{mygray}\; +L-Reg & 40.91        & 68.17        & 17.05        & \blue{26.60}        & \blue{49.11}        & 11.71        & 36.82        & 60.76        & 12.85              \\
\rowcolor{mygray}\; +\textbf{PL-Reg} & {\blue{42.38}}&	69.82&	{\blue{18.34}}&	26.03&	46.85&	{\blue{12.26}}&	{\blue{38.01}}&	{\blue{63.05}}&	12.94 
\\
\rowcolor{mygray}\; Imp.&\red{1.46}&	\red{1.66}&	\red{1.30}&	-0.57&	-2.26& \red{0.55}&	\red{1.19}&	\red{2.9}&	\red{0.09}
\\
\hline
 & \multicolumn{3}{c|}{\textbf{quickdraw}}    & \multicolumn{3}{c|}{\textbf{real}}      &
\multicolumn{3}{c}{\textbf{sketch}}
\\ Method & {All}  & {Known}  & {Unknown} & {All}  & {Known}  & {Unknown} & {All}  & {Known}  & {Unknown}        \\ 
\hline
ERM & 8.88         & 12.83        & 4.91         & 17.88        & 31.20        & 4.10         & 29.01        & 56.45        & 8.76         \\
PIM & \blue{9.92}         & \blue{14.73}        & \blue{5.09}         & 29.09        & 53.88        & 3.42         & 32.12        & 59.35        & 12.03        \\
MIRO & 8.06         & 12.12        & 3.98         & 42.19        & \blue{75.49}        & 7.72         & 34.83        & \blue{66.51}        & 11.46        \\
\hline
GMDG & 7.43         & 11.83        & 3.01         & 42.84        & 75.27        & 9.27         & 35.01        & 66.95        & 11.46        \\
\rowcolor{mygray}\; +L-Reg & 9.11         & 13.51        & 4.70         & 42.63        & 74.42        & 9.72         & 34.44        & 65.13        & 11.80        \\
\rowcolor{mygray}\; \textbf{+PL-Reg}& 
8.60& 	13.88& 	3.30& 	{\blue{44.09}}& 	74.47& 	{\blue{12.64}}& 	{\blue{36.12}}& 	65.09& 	{\blue{14.76}}
\\
\rowcolor{mygray}\; Imp.& -0.51& 	\red{0.38}& 	-1.40	& \red{1.46}& 	\red{0.05}& 	\red{2.92}& 	\red{1.69}& 	-0.04& 	\red{2.96}
\\ 
\hline
\end{tabular}%
}
\end{table*}

\subsubsection{Overall results}
 \cref{tab:avg_mdg_gcd_more_results} and  \cref{tab:MDG+GCD_res} overall showcases PL-Reg's effectiveness in tackling the challenge of mDG+GCD. If we take a closer examination,  GMDG+PL-Reg achieves an accuracy of $55.04\%$ for all the classes, the highest among all the methods. These gains highlight PL-Reg’s balanced ability to handle both known and unknown classes effectively, making it the most versatile method in this comparison.

Similar to the phenomenon in GCD tasks, GMDG+PL-Reg achieves $68.37\%$ accuracy for known classes, slightly worse than GMDG+L-Reg~($69.87\%$). This minor trade-off is compensated by its significant improvements in other metrics (e.g., unknown class accuracy). Please refer to \cref{Discussion} for more discussions. 

In terms of unknown classes, GMDG+PL-Reg delivers an unknown class accuracy of $34.43\%$, the highest in the table and substantially better than the previous methods, especially GMDG+L-Reg, demonstrating that PL-Reg is highly effective at discovering and classifying novel categories. The improvements brought by PL-Reg on unknown classes again validate \cref{prop:partial_better_generaliztion}.
The detailed analysis of each dataset is provided in \cref{sec:mDG_GCD_detialed_analysis}. 

\subsubsection{Results across datasets}
\label{sec:mDG_GCD_detialed_analysis}
The results across different domains in various datasets reflect the improved generalization of the model with PL-Reg for unknown classes, as theoretically shown in \cref{prop:partial_better_generaliztion}.  
However, one may already notice that the improvements brought by PL-Reg may not be consistent across all domains for all datasets, though the averaged improvements are evident. This is due to the \textit{no free lunch} theory, where the averaged improvements across all domains may slightly compromise the performance in some specific domains.  

\textbf{PACS dataset.}
In the PACS dataset (as shown in \cref{tab:avg_mdg_gcd_more_results}), PL-Reg generates the highest all classes accuracy of $74.19\%$, significantly outperforming GMDG+L-Reg by $+6.39\%$. This improvement is attributed to substantial gains in unknown class accuracy ($42.69\%$, an increase of $+11.36\%$ compared to GMDG+L-Reg), and known class accuracy ($92.07\%$) is slightly improved compared to GMDG+L-Reg ($91.86\%$). 
Across PACS domains (Art Painting, Cartoon, Photo, Sketch as shown in \cref{tab:mdg_gcd_more_results1}),
GMDG+PL-Reg achieves significant improvements for unknown classes, particularly in domains with high variability, such as Photo and cartoon, boosting evident gains of $+40.71\%$ and $+15.20\%$ in comparison to L-Reg, respectively. 
One may notice in  \cref{tab:mdg_gcd_more_results1} that using GMDG as a backbone may lead to poor performance in the sketch domain.  To further validate PL-Reg, we also apply it to the MIRO backbone on the PACS datasets. Consistent improvements across all unknown classes for all unseen domains can be observed with MIRO+PL-Reg, indicating the efficacy of our proposed method.


\textbf{HomeOffice dataset.}
As observed  in \cref{tab:avg_mdg_gcd_more_results}, GMDG+PL-Reg exhibits significant improvements in the HomeOffice dataset, achieving the highest accuracy of $54.73\%$ for all the classes, with gains of $+2.78\%$ over GMDG+L-Reg. The unknown class accuracy improves to $22.58\%$, an increase of $+4.44\%$ compared to GMDG+L-Reg. Known class accuracy is also competitive at $80.98\%$ with a slight improvement from GMDG+L-Reg $79.74\%$, demonstrating PL-Reg's ability to perform consistently excellent across domains with both known and unknown categories. 
\cref{tab:mdg_gcd_more_results2} further shows that
GMDG+PL-Reg leads to consistent gains across all domains in HomeOffice, achieving the highest improvements in unknown class accuracy in the real-world domain $30.86\%$ with an increase of $+13.27\%$ from GMDG+L-Reg.

\textbf{VLCS dataset.} 
In the VLCS dataset, GMDG+PL-Reg accomplishes an accuracy of $65.10\%$ for all classes, marking an improvement of $+2.78\%$ over GMDG+L-Reg (please refer to \cref{tab:avg_mdg_gcd_more_results}). This gain is also largely due to improvements in unknown class accuracy ($45.87\%$, an increase of $+9.78\%$ in comparison with GMDG+L-Reg). Known class performance remains competitive at $80.13\%$ with a slight drop from GMDG+L-Reg.
Specifically, as shown in \cref{tab:mdg_gcd_more_results3}, PL-Reg outperforms L-Reg across most domains, with significant gains in unknown class accuracy in SUN09 and VOC2007~($+27.58\%$  and $+15.95\%$ form GMDG+L-Reg, respectively). These results suggest PL-Reg’s capability to generalize effectively across diverse domains and ensure consistent classification accuracy for unknown classes.

\textbf{TerraIncognita dataset.}
\cref{tab:avg_mdg_gcd_more_results} shows that
GMDG+PL-Reg delivers the second-best accuracy of $8.93 \%$ for all the classes, a notable improvement of $+3.08\%$ compared to GMDG+L-Reg. Concretely,
unknown class accuracy improves significantly to $54.28\%$, reflecting an increase of $+12.73\%$.
Particularly, \cref{tab:mdg_gcd_more_results4} indicates that 
PL-Reg excels in TerraIncognita’s challenging domains, with the largest improvements in unknown class accuracy observed in Local 100, Local 43, and Local 38 ($+5.04\%$, $+22.08\%$ and $+27.35\%$ from GMDG+PL-Reg, respectively). This phenomenon is consistent with the aforementioned datasets. 

\textbf{DomainNet dataset.}
The DomainNet dataset is a large-scale and diverse dataset consisting of six domains:
clipart, info, painting, quickdraw, real, and sketch. 
The results of DomianNet are exhibited in \cref{tab:avg_mdg_gcd_more_results,tab:mdg_gcd_more_results5}.
In the DomainNet dataset, GMDG+PL-Reg achieves an accuracy of $32.54\%$ for all datasets, an improvement of $+0.79\%$ over GMDG+L-Reg. The method shows evident gains in unknown class accuracy ($12.37\%$, an increase of +$61.07\%$ from GMDG+L-Reg) while maintaining the best performance at $55.53\%$ for known classes. The improvements across all metrics indicate PL-Reg’s capacity to generalize well in complex, large-scale datasets with diverse domains.
For each specific domain (see \cref{tab:mdg_gcd_more_results5}).
GMDG+PL-Reg demonstrates consistent effectiveness across the diverse domains of the DomainNet dataset.
It excels particularly in the clipart, painting, and real domains, where boosts in both known and unknown class accuracies are observed.
For domains like info and sketch, PL-Reg still lifts up the performance of the unknown classes. 
Overall, PL-Reg maintains robust performance across diverse domains, showcasing its adaptability and generalization capabilities in the large-scale dataset.

\subsubsection{Visualizations of PL-Reg on mDG+GCD tasks}

\begin{figure*}[t]
\centering
\includegraphics[width=\linewidth]{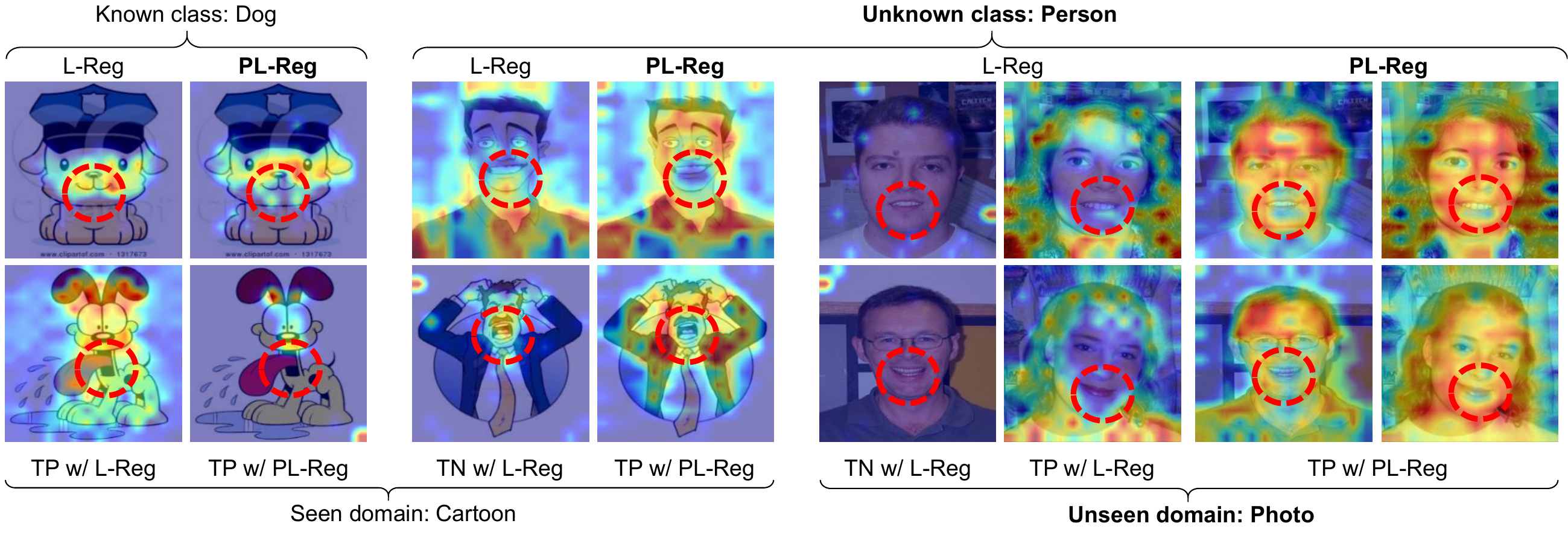}
\caption{
MDG+GCD visualization: GradCAM heat maps visualization of GMDG+L-Reg and GMDG+PL-Reg that trained on PACS datasets while the unseen domain is the photo, respectively. The semantics around mouth areas learned by using the sentential logical method (L-Reg) may lead to misclassified samples and generalization degradation for unknown classes. 
}\label{fig:grad_vis}
\end{figure*}

\begin{figure*}[t]
\centering
\includegraphics[width=\linewidth]{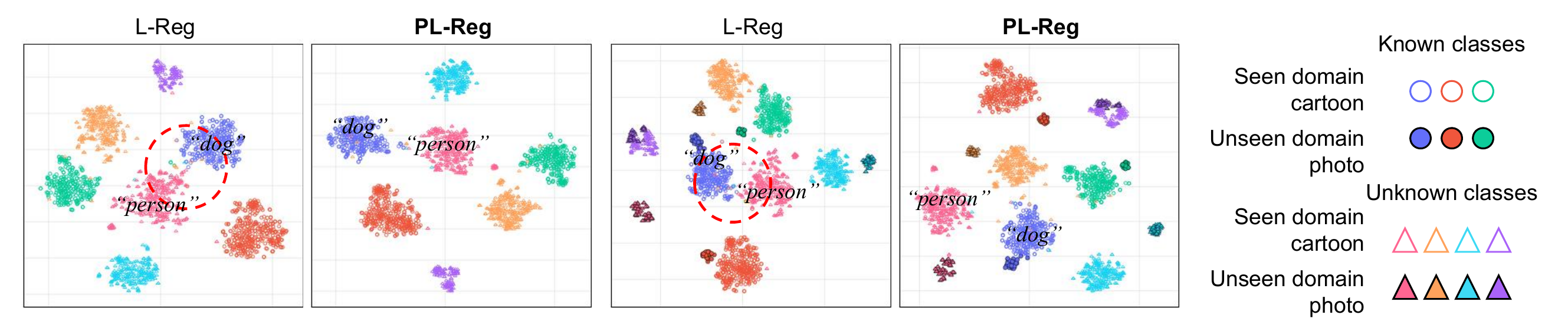}
\caption{
MDG+GCD visualization: Visualization features of GMDG+L-Reg and GMDG+PL-Reg that are trained on PACS datasets while the unseen domain is the photo, respectively.
}\label{fig:tsne}
\end{figure*}

The GradCAM visualization in \cref{fig:grad_vis} reveals critical differences between L-Reg and PL-Reg in learning features for the MDG+GCD task across both seen and unseen domains. 

The model trained with L-Reg exhibits a significant bias towards known classes and seen domains, yielding generalization degradation of unknown classes from unseen domains. 
For instance, in the cartoon domain, the image samples of animals, such as dogs, are usually humanized and have rich facial expressions. 
The heat maps for cartoon dogs indicate that L-Reg focuses heavily on regions like the mouth, a feature that semantically overlaps with cartoon humans. 
Using such overlapped features for classification ultimately results in incorrect classification of unknown classes, such as person classes, due to the model's learned confusion between these shared features.
The overlapped features remain in the model due to L-Reg's grounding sentential logic theory, leading to a leaning towards all features that would be used for classification.

In contrast, PL-Reg demonstrates a more nuanced and class-specific focus, avoiding the overlapping of human-like features in the cartoon domain. For known classes like the cartoon dog, PL-Reg highlights a broader and more context-aware set of features, reducing reliance on human-specific regions such as the mouth. 
Such a phenomenon is triggered by the partial logic-based design, allowing models to maintain undefined formulations. 
As a result, PL-Reg is less likely to misclassify unknown classes like the cartoon person.

Furthermore, in the unseen domain (photo), the features learned by L-Reg for unknown classes in the cartoon domain fail to generalize effectively.
As shown in \cref{fig:grad_vis}, the failed cases are seen in the GradCAM visualizations where attention is scattered and misaligned, while those correctly classified images are those whose middle facial features, such as areas around mouths, are not used. 
In contrast, PL-Reg retains its generalization ability by maintaining broader, more transferable feature representations, enabling it to accurately localize and classify unknown classes like ``person" in the photo domain. Specifically, it can be noticed that the mouth areas are consistently ignored across domains and classes. 

\cref{fig:tsne} further varies the confusion between the known dog and unknown person classes across seen cartoon and unseen photo domains. It can be noticed that, with PL-Reg, clusters of those two categories are more scattered, while L-Reg may lead to a more ambiguous decision boundary.

Combined with quantitative results of mDG+GCG, PL-Reg's advantage in both seen and unseen domains enables robust generalization by avoiding overfitting to shared, domain-specific features. This empirical analysis also supports our main theory in \cref{prop:partial_better_generaliztion}.

\subsection{Experiments on incremental learning}
\label{Experiments on incremental learning}

\subsubsection{Experimental settings}

\begin{table}[t]
\caption{CIL experimental settings: Hyper-parameters of applying PL-Reg to UCIR+LWS.}\label{tab:CIL_hyper_LWS}
\centering
\begin{tabular*}{\linewidth}{@{\extracolsep\fill}lccc}
  & \multicolumn{3}{c}{\textbf{Ordered Long-tailed}} \\  \hline
 & $\omega_{p1}$ & $\omega_{p2}$ &$\omega_{L-Reg}$ \\ \hline
CIFAR100 & 1e-4  & 1e-4 & 1e-3 \\ 
Imagenet-Subset & 1e-3  & 1e-3  & 5e-3 \\ 
\hline
   & \multicolumn{3}{c}{\textbf{Shuffled Long-tailed}} \\  \hline
    & $\omega_{p1}$ & $\omega_{p2}$ &$\omega_{L-Reg}$ \\ \hline
CIFAR100 &  1e-3 &  1e-3 & 1e-3 \\ 
Imagenet-Subset &  1e-3 &  1e-3 & 1e-3\\ 
 \hline
\end{tabular*}%
\end{table}

\begin{table}[t]
\caption{CIL experimental settings: Hyper-parameters of applying PL-Reg to UCIR+GVAlign.}\label{tab:CIL_hyper_GVAlgin}
\centering
\begin{tabular*}{\linewidth}{@{\extracolsep\fill}lccc}
  & \multicolumn{3}{c}{\textbf{Ordered Long-tailed}} \\ \hline
 & $\omega_{p1}$ & $\omega_{p2}$ &$\omega_{L-Reg}$ \\ \hline
CIFAR100 &  1e-4 & 1e-4 & 1e-2 \\ %
Imagenet-Subset &  1e-4 & 1e-4 & 1e-2 \\ %
\hline
   & \multicolumn{3}{c}{\textbf{Shuffled Long-tailed}} \\ \hline
    & $\omega_{p1}$ & $\omega_{p2}$ &$\omega_{L-Reg}$ \\ \hline
CIFAR100 &  1e-3 & 1e-3 & 1e-2 \\ %
Imagenet-Subset &  1e-3 & 1e-3 & 5e-3\\ %
 \hline
\end{tabular*}%
\end{table}

\begin{table*}[t]
\caption{Results of CIL: Results of L-Reg and proposed PL-Reg applied to various CIL backbones on the CIFAR100 dataset. The best results of each backbone are highlighted in \blue{blue}. $^*$ denotes reproduced results.}
\label{tab:cil_cifar100}
\resizebox{\textwidth}{!}{%
\begin{tabular}{lcccccccccccc}
\hline
 & \multicolumn{12}{c}{\textbf{CIFAR100 Ordered Long-tailed}} \\ 
Session & 0 & 1 & 2 & 3 & 4 & 5 & 6 & 7 & 8 & 9 & \multicolumn{1}{l|}{10} & \multicolumn{1}{c}{Avg.} \\ \hline
UCIR $^*$ & 59.7 & 54.2 & 47.7 & 45.9 & 43.2 & 39.3 & 37.7 & 36.0 & 34.8 & 31.8 & \multicolumn{1}{l|}{21.6} & 41.1 \\ \hline
UCIR+LWS$^*$ & 60.1 & 57.5 & 50.0 & 48.8 & 46.7 & 41.8 & 39.9 & 38.4 & 36.5 & 33.5 & \multicolumn{1}{l|}{30.7} & 44.0 \\
\rowcolor{mygray}{ \;+ L\_reg} & 59.2 & 55.5 & 50.2 & 47.7 & 46.2 & 42.2 & 39.7 & 37.7 & 36.0 & 32.6 & \multicolumn{1}{l|}{30.3} & 43.4 \\
\rowcolor{mygray}{ \;\textbf{+ PL\_reg}} &\blue{\textbf{61.0}}&\blue{\textbf{59.1}}&\blue{\textbf{53.4}}&\blue{\textbf{51.8}}&\blue{\textbf{48.6}}&\blue{\textbf{42.8}}&\blue{\textbf{41.2}}&\blue{\textbf{38.8}}&\blue{\textbf{37.4}}&\blue{\textbf{34.4}}&\multicolumn{1}{l|}{\blue{\textbf{31.8}}}&\blue{\textbf{45.5}} \\
\hline
UCIR + GVAlign$^*$ & \blue{66.1} & 60.3 & 54.2 & 53.1 & 51.6 & 45.9 & 42.2 & 40.3 & 39.3 & 37.3 & \multicolumn{1}{l|}{28.7} & 47.2 \\
\rowcolor{mygray}{ \;+ L\_reg} & 64.7 & 62.1 & 54.1 & 52.5 & 49.3 & 46.7 & 44.8 & \blue{43.1} & \blue{41.5} & \blue{38.5} & \multicolumn{1}{l|}{28.3} & 47.8 \\
\rowcolor{mygray}{ \;\textbf{+ PL\_reg}}  & \textbf{65.1} & \blue{\textbf{62.5}} & \blue{\textbf{54.6}} & \blue{\textbf{55.1}} & \blue{\textbf{52.4}} & \blue{\textbf{47.8}} & \blue{\textbf{45.2}} & \textbf{42.9} & \textbf{41.0} & \textbf{37.7} & \multicolumn{1}{l|}{\blue{\textbf{29.8}}} & \blue{\textbf{48.6}} \\ \hline
 & \multicolumn{12}{c}{\textbf{CIFAR100 Shuffled Long-tailed}} \\ 
Session & 0 & 1 & 2 & 3 & 4 & 5 & 6 & 7 & 8 & 9 & \multicolumn{1}{l|}{10} & \multicolumn{1}{c}{Avg.} \\ \hline
UCIR $^*$ & 39.9 & 36.7 & 31.1 & 33.2 & 34.9 & 31.9 & 33.0 & 33.3 & 31.9 & 30.6 & \multicolumn{1}{l|}{31.3} & 33.4 \\ \hline
UCIR+LWS$^*$ & 44.7 & 42.4 & 37.0 & 39.2 & 39.3 & 37.5 & 37.0 & 37.0 & {36.7} & 34.2 & \multicolumn{1}{l|}{34.6} & 38.1 \\
\rowcolor{mygray}{ \;+ L\_reg} & \blue{44.9} & {45.5} & 37.4 & {39.1} & 40.4 & 37.7 & 36.5 & 36.7 & 36.2 & {34.7} & \multicolumn{1}{l|}{{36.0}} & 38.6 \\
\rowcolor{mygray}{ \;\textbf{+ PL\_reg}} &44.6& \blue{\textbf{46.9}}&\blue{\textbf{38.7}}&\blue{\textbf{41.4}}&\blue{\textbf{41.0}}&\blue{\textbf{39.0}}&\blue{\textbf{39.0}}&\blue{\textbf{38.4}}&\blue{\textbf{37.6}}&\blue{\textbf{35.9}}&\multicolumn{1}{l|}{\blue{\textbf{36.5}}}&\blue{\textbf{39.9}}
\\
\hline
UCIR + GVAlign$^*$ & 47.3 & \blue{47.5} & 40.0 & 42.1 & \blue{43.8} & \blue{43.0} & 40.1 & \blue{39.7} & 39.1 & 37.0 & \multicolumn{1}{l|}{\blue{38.5}} & 41.6 \\
\rowcolor{mygray}{ \;+ L\_reg} & \blue{47.9} & 46.8 & 40.8 & \blue{43.1} & 43.5 & 42.4 & 40.6 & 39.6 & 39.3 & \blue{37.5} & \multicolumn{1}{l|}{37.1} & 41.7 \\
\rowcolor{mygray}{ \;\textbf{+ PL\_reg}}  & \textbf{47.1} & \textbf{47.3} & \blue{\textbf{42.1}} & \textbf{42.0} & \textbf{43.4} & \textbf{42.4} & \blue{\textbf{41.4}} & \blue{\textbf{39.7}} & \blue{\textbf{39.8}} & \textbf{37.2} & \multicolumn{1}{l|}{\textbf{38.1}} & \blue{\textbf{41.9}} \\ \hline
\end{tabular}%
}
\end{table*}

\begin{table*}[t]
\caption{Results of CIL: Results of L-Reg and proposed PL-Reg applied to various CIL backbones on the ImageNet-Subset dataset. The best results of each backbone are highlighted in \blue{blue}. $^*$ denotes reproduced results.}
\label{tab:cil_imagenet}
\resizebox{\textwidth}{!}{%
\begin{tabular}{lcccccccccccc}
\hline
 & \multicolumn{12}{c}{\textbf{ImageNet-Subset Ordered Long-tailed}} \\ 
\multicolumn{1}{l}{Session} & 0 & 1 & 2 & 3 & 4 & 5 & 6 & 7 & 8 & 9 & \multicolumn{1}{c|}{10} & Avg. \\ \hline
UCIR$^*$ & 69.1 & 63.6 & 61.6 & 57.5 & 52.0 & 51.1 & 45.4 & 45.0 & 39.7 & 38.2 & \multicolumn{1}{c|}{37.5} & 51.0 \\ \hline
{UCIR+LWS$^*$} & \blue{71.4} & {64.1} & {61.8} & \blue{61.0} & {54.2} & {52.7} & {46.4} & {45.9} & {43.2} & {41.7} & \multicolumn{1}{c|}{{41.3}} & { 53.1} \\
\rowcolor{mygray}{ \;+ L\_reg} & {71.1} & {61.7} & {59.9} & {59.7} & {54.6} & {52.7} & { 46.4} & { 45.7} & \blue{43.3} & {41.8} & \multicolumn{1}{c|}{{41.3}} & { 52.6} \\ 
\rowcolor{mygray}{\;\textbf{+ PL\_reg}} &71.0&\blue{\textbf{65.2}}&\blue{\textbf{62.9}}&58.1&\blue{\textbf{55.4}}&\blue{\textbf{53.6}}&\blue{\textbf{47.4}}&\blue{\textbf{48.0}}&\blue{\textbf{43.3}}&\blue{\textbf{42.4}}&\multicolumn{1}{c|}{\blue{\textbf{41.9}}}&\blue{\textbf{53.6}}\\
\hline
{UCIR + GVAlign$^*$} & {72.6} & {68.7} & {65.8} & {61.7} & {56.3} & {51.6} & {48.0} & {48.0} & {44.5} & {43.3} & \multicolumn{1}{c|}{{42.4}} & {54.8} \\
\rowcolor{mygray}{\; + L\_reg} & {73.5} & \blue{70.4} & \blue{66.6} & {62.6} & {59.0} & {54.5} & {49.4} & {49.4} & {47.3} & {44.4} & \multicolumn{1}{c|}{{43.5}} & {56.4} \\
\rowcolor{mygray}{\;\textbf{+ PL\_reg}} & \blue{\textbf{74.0}} & {\textbf{68.4}} & {\textbf{66.2}} & \blue{\textbf{64.2}} & \blue{\textbf{60.2}} & \blue{\textbf{55.6}} & \blue{\textbf{50.3}} & \blue{\textbf{50.1}} & \blue{\textbf{47.8}} & \blue{\textbf{45.5}} & \multicolumn{1}{c|}{\blue{\textbf{43.8}}} & \blue{\textbf{56.9}} \\ \hline
\multicolumn{1}{c}{} & \multicolumn{12}{c}{\textbf{ImageNet-Subset Shuffled Long-tailed}} \\ 
\multicolumn{1}{l}{Session} & 0 & 1 & 2 & 3 & 4 & 5 & 6 & 7 & 8 & 9 & \multicolumn{1}{c|}{10} & Avg. \\ \hline
UCIR$^*$ & 52.4 & 54.8 & 56.8 & 50.6 & 46.6 & 44.4 & 37.9 & 39.9 & 37.5 & 36.2 & \multicolumn{1}{c|}{36.8} & 44.9 \\ \hline
UCIR+LWS$^*$ & 58.3 & 53.4 & 55.3 & 53.0 & 47.8 & 46.5 & 41.1 & 42.5 & 40.8 & 39.6 & \multicolumn{1}{c|}{39.4} & 47.1 \\
\rowcolor{mygray}{\;+ L\_reg} & 57.5 & \blue{56.7} & \blue{56.5} & 53.7 & {50.3} & 47.0 & 41.7 & 43.7 & 41.7 & 40.7 & \multicolumn{1}{c|}{{40.8}} & 48.2 \\
\rowcolor{mygray}{\;\textbf{+ PL\_reg}}&\blue{\textbf{58.6}}&\textbf{55.8}&\textbf{56.2}&\blue{\textbf{54.7}}&\blue{\textbf{51.2}}&\blue{\textbf{49.3}}&\blue{\textbf{43.4}}&\blue{\textbf{44.1}}&\blue{\textbf{42.4}}&\blue{\textbf{41.6}}& \multicolumn{1}{c|}{\blue{\textbf{41.5}}}&\blue{\textbf{49.0}}\\
\hline
UCIR + GVAlign$^*$ & \blue{60.6} & 58.1 & \blue{60.4} & 52.6 & 48.0 & 47.1 & 41.1 & 43.5 & \blue{41.3} & 40.6 & \multicolumn{1}{c|}{41.1} & 48.6 \\
\rowcolor{mygray}{\;+ L\_reg}  & 59.6 & 59.2 & 58.9 & 52.6 & \blue{50.0} & 47.3 & 40.7 & 43.7 & 40.7 & \blue{41.2} & \multicolumn{1}{c|}{\blue{41.2}} & 48.6 \\
\rowcolor{mygray}{\;\textbf{+ PL\_reg}} & \textbf{60.3} & \blue{\textbf{59.4}} & \textbf{60.1} & \blue{\textbf{53.9}} & \textbf{49.9} & \blue{\textbf{48.3}} & \blue{\textbf{42.8}} & \blue{\textbf{44.7}} & \textbf{40.9} & \textbf{40.3} & \multicolumn{1}{c|}{\textbf{40.7}} & \blue{\textbf{49.2}} \\ \hline
\end{tabular}%
}
\end{table*}

We conduct experiments on two class incremental learning settings based on long-tailed distributions, i.e., \textit{Ordered} Long-tailed CIL and \textit{Shuffled} Long-tailed CIL~\cite{liu2022long}.  \textit{Ordered} Long-tailed CIL presents a scenario where all classes are arranged based on the number of samples per class, and these classes are subsequently divided into sequential tasks. \textit{Shuffled} Long-tailed CIL randomly assigns classes across tasks, where the data in each task follows an imbalanced distribution. \textit{Shuffled} Long-tailed CIL can be considered a more challenging yet realistic case in the real world.


\textbf{Competitors.}
We compare our approach against two prevailing  incremental learning methods, LWS~\cite{liu2022long} and GVAlign~\cite{kalla2024robust}\footnote{We reproduce their results based on the official released repository for fair comparisons.}.
By integrating these methods into the baseline model UCIR~\cite{hou2019learning}, we evaluate and compare their average performance across 11 incremental tasks.

\textbf{Datasets.} 
Following prior studies~\cite{kalla2024robust, liu2022long}, we conduct experiments on CIFAR100~\cite{krizhevsky2009learning} and the ImageNet-Subset (100 classes)~\cite{krizhevsky2012imagenet}. Each dataset is divided into 11 incremental tasks, with the first task containing 50 classes and the remaining 10 tasks consisting of 5 classes each. The classes across tasks are mutually exclusive. The training data is highly imbalanced, with a data imbalance ratio of $0.01$, while the testing data is balanced to ensure fair evaluation.

\textbf{Evaluation metric.}
Following previous works~\cite{kalla2024robust, liu2022long}, we evaluate the average accuracy across 11 incremental tasks.

\textbf{Training details.}
To validate the effectiveness of PL-Reg, we parallelly produced the results of L-Reg by applying them to previous UCIR+LWS and UCIR+GVAlign, respectively. 
In terms of PL-Reg, partial logic masks, and classifications are applied to the latent features of the visual encoders, and these masked latent features are subsequently utilized for the application of L-Reg.
The sole L-Reg is applied to the same latent features but without partial masks and classifications. 
The parameters of the layers of generating partial logic masks are initialized for each session. 
The models are trained with the aforementioned datasets following the ordered and shuffled long-tailed CIL settings~\cite{kalla2024robust, liu2022long}. 
Due to the added regularization, we extend the second stage of fine-tuning in UCIR to $50$ epochs to ensure convergence of the model.  
The losses proposed in UCIR+LWS or UCIR+GVAlign are treated as 
$L_{main}$. The values of hyper-parameters used for training are exhibited in \cref{tab:CIL_hyper_LWS,tab:CIL_hyper_GVAlgin}, respectively.


\begin{figure*}[t]
\centering
\includegraphics[width=\linewidth]{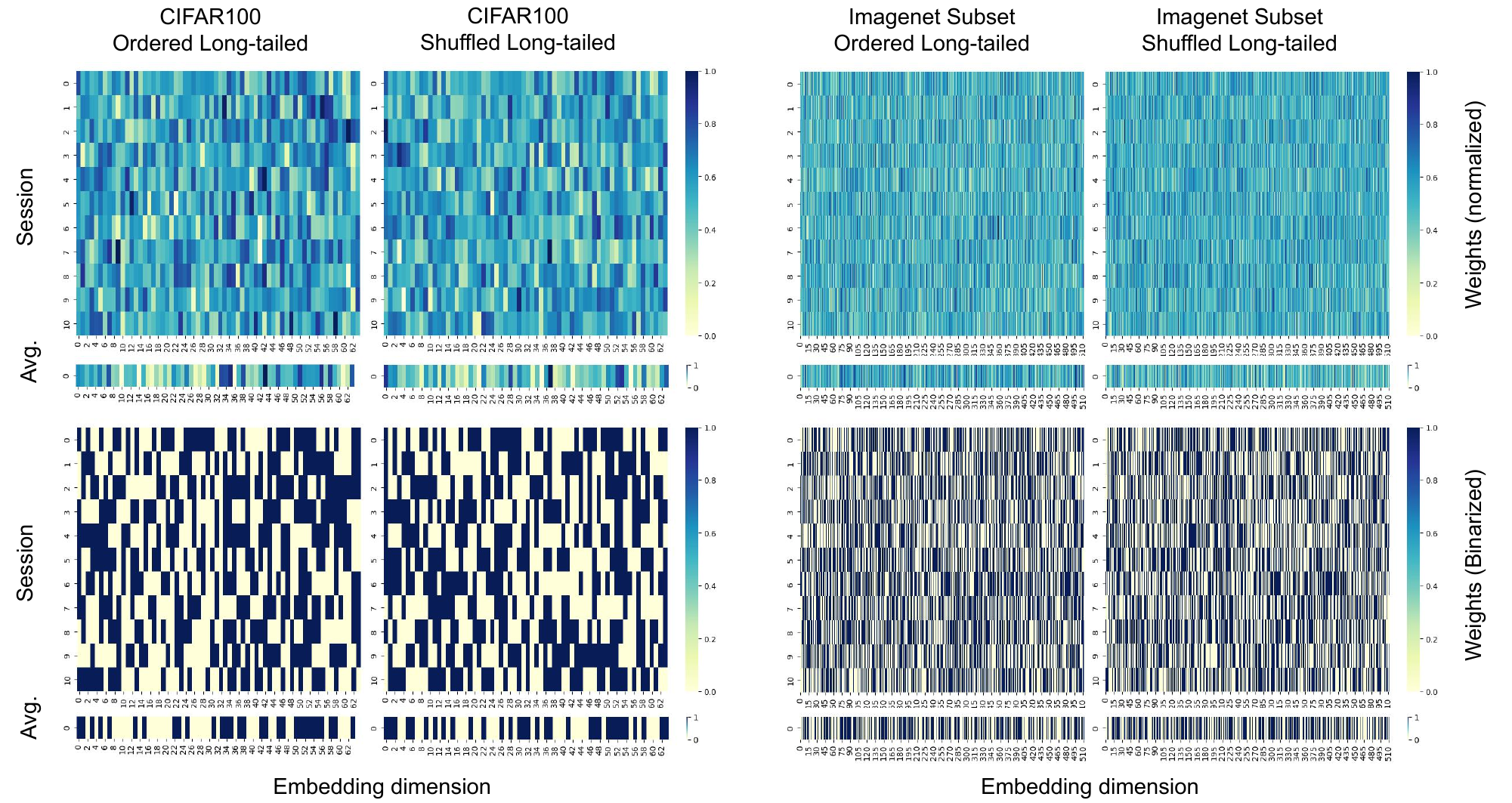}
\caption{
CIL analysis:
Visualizations of partial logic mask generator weights of each CIL session of LUCIR+LWS+PL-Reg. The top row displays heat maps of normalized weights, while the bottom row shows binarized normalized weights, with values greater than 0.5 set to $1$ and others set to $0$. Avg. denotes the averaged weights across all sessions.
}\label{fig:CIL_weights}
\end{figure*}

\subsubsection{Results of CIFAR100}
As shown in \cref{tab:cil_cifar100},
PL-Reg consistently improves average accuracy and session-specific performance across both CIFAR100 settings (\textit{Ordered} and \textit{Shuffled} long-tail CIL), while L-Reg may not, reinforcing its effectiveness in enhancing existing CIL backbones. 
These improvements again validate our \cref{prop:partial_better_generaliztion}.

\textbf{Results of CIFAR100: Ordered long-tailed.}
For the ordered long-tailed setting of the CIFAR100 dataset, using L-Reg may not certify an improvement from the bare baseline, e.g., a slight decrease is observed while applying L-Reg to UCIR+LWS for ordered long-tailed CIL.
Moreover, PL-Reg consistently outperforms L-Reg in improving the performance of both UCIR+LWS and UCIR+GVAlign backbones. For UCIR+LWS, PL-Reg increases the average accuracy from $43.4\%$ (solely with L-Reg) to $45.5\%$, demonstrating incremental improvements across all sessions. 
Similarly, for UCIR+GVAlign, PL-Reg improves the average accuracy from $47.8\%$ (solely with L-Reg) to $48.6\%$, with notable gains in earlier sessions and the last sessions. 
These results highlight the superiority of PL-Reg over L-Reg in the transferability of adapting to unknown classes. 

\textbf{Results of CIFAR100: Shuffled long-tailed.}
On CIFAR100, under the shuffled long-tailed setting, both PL-Reg and L-Reg improve the performance of two baselines.
PL-Reg still demonstrates better effectiveness than L-Reg in addressing class imbalance and enhancing generalization. For UCIR+LWS, PL-Reg improves the average accuracy from $38.6\%$ (with L-Reg) to $39.9\%$. For UCIR+GVAlign, PL-Reg raises the average accuracy from $41.7\%$ (with L-Reg) to $41.9\%$. 
These results indicate that PL-Reg surpasses L-Reg in stabilizing performance under shuffled distributions.

\begin{table*}[t]
\caption{Sensitive analysis of hyper-parameters: Results of Ordered long-tailed CIL task on ImageNet-Subset dataset. The final parameters adopted are highlighted.}
\label{tab:sensitive_analysis}
\resizebox{\textwidth}{!}{%
\begin{tabular}{c|ccc|c|ccccccccccc}
 \hline
\multicolumn{1}{l}{} & \multicolumn{1}{l}{} & \multicolumn{1}{l}{} & \multicolumn{1}{l}{} & \multicolumn{1}{l}{} & \multicolumn{11}{c}{Sessions} \\ \hline
UCIR+LWS & $w_{p1}$ & $w_{p2}$ & $w_{L-Reg}$ & Avg. & 0 & 1 & 2 & 3 & 4 & 5 & 6 & 7 & 8 & 9 & 10 \\ \hline
+L-Reg & - & - & 5e-3 & 52.6 & 71.1 & 61.7 & 59.9 & 59.7 & 54.6 & 52.7 & 46.4 & 45.7 & 43.3 & 41.8 & 41.3 \\ \hline
\rowcolor{mygray}
\textbf{+PL-Reg (Final)} & {1e-3} & 1e-3 & 5e-3 & 53.6 & 71.0 & 65.2 & 62.9 & 58.1 & 55.4 & 53.6 & 47.4 & 48.0 & 43.3 & 42.4 & 41.9 \\
 & \textbf{1e-3} & 5e-4 & 5e-3 & 53.2 & 71.5 & 67.1 & 63.5 & 60.0 & 54.4 & 51.4 & 46.1 & 47.0 & 42.1 & 41.4 & 40.5 \\
 & \textbf{1e-3} & 2e-3 & 5e-3 & 53.2 & 72.2 & 64.9 & 61.9 & 59.3 & 53.9 & 52.5 & 46.2 & 47.0 & 43.8 & 41.7 & 41.5 \\ 
 & 5e-4 & \textbf{1e-3} & 5e-3 & 52.2 & 71.2 & 62.7 & 61.1 & 57.3 & 53.9 & 51.4 & 45.1 & 45.8 & 42.8 & 41.4 & 41.2 \\
\multirow{-4}{*}{+PL-Reg} & 2e-3 & \textbf{1e-3} & 5e-3 & 53.3 & 71.6 & 66.0 & 60.8 & 59.6 & 53.9 & 53.0 & 46.9 & 47.3 & 43.0 & 42.1 & 42.1 \\ \hline
\end{tabular}%
}
\end{table*}

\subsubsection{Results of ImageNet-Subset}
Similar to CIFAR100, PL-Reg consistently outperforms L-Reg by providing steady improvements in average accuracy and session-wise performance across both ImageNet-Subset settings,  

\textbf{Results of ImageNet-Subset: Ordered long-tailed.}
On the ImageNet-Subset Ordered Long-tailed dataset, PL-Reg exhibits clear advantages over L-Reg by further boosting performance across sessions. For UCIR+LWS, PL-Reg improves the average accuracy from $52.6\%$ (with L-Reg) to $53.6\%$. Similarly, for UCIR+GVAlign, PL-Reg raises the average accuracy from $56.4\%$ solely with L-Reg to $56.9\%$, specifically achieving substantial improvements in mid-to-late sessions.

\textbf{Results of ImageNet-Subset: Shuffled long-tailed.}
On the ImageNet-Subset Shuffled Long-tailed dataset, PL-Reg again demonstrates superiority over L-Reg. For UCIR+LWS, PL-Reg increases the average accuracy from $48.2\%$ (with L-Reg) to $49.0\%$, with increases observed for almost all sessions. 
For UCIR+GVAlign, PL-Reg achieves a more significant performance boost, improving the average accuracy from $49.2\%$ (with L-Reg) to $48.6\%$. 
These results highlight PL-Reg's ability to handle class imbalance and stabilize performance in comparison to L-Reg, particularly in more challenging shuffled settings.

\subsubsection{Visualizations of PL-Reg on CIL tasks}

\cref{fig:CIL_weights} presents visualizations of the weights of the partial logic mask generator for different CIL sessions across different datasets under both ordered long-tailed and shuffled long-tailed settings. The top row displays heat maps of normalized weights, while the bottom row illustrates binarized weights, where values above $0.5$ are set to $1$, and others are set to $0$. 
The bottom row highlights the differences of weights across sessions.

In all four settings, the distinct weight patterns across sessions suggest that the mask generator allows the model to capture unique feature subsets for each session. 
This weighting pattern leads to session-specific masks,
saving different features for different sessions, which is especially beneficial for mitigating catastrophic forgetting in CIL under both long-tailed settings (please refer to results in quantitative results in \cref{tab:cil_cifar100,tab:cil_imagenet}).
Moreover, the averaged weights across all sessions exhibit variations for each dimension, suggesting that there is still reserved room for future incremental classes. 

Overall, the distinctiveness of the partial logic mask generator's weights across sessions under both ordered and shuffled long-tailed settings confirms its effectiveness in facilitating feature retention and adaptability for unknown classes, which supports our \cref{prop:partial_better_generaliztion}.

\subsubsection{Sensitive analysis}

\cref{tab:sensitive_analysis} presents the results of the Ordered Long-Tailed Class Incremental Learning (CIL) task on the ImageNet-Subset dataset, showing the impact hyper-parameter settings of $w_{p1}, w_{p2}, w_{L-Reg}$ applied to baseline UCIR+LWS. The results of UCIR+LWS+L-Reg are exhibited for better comparison. 

The results shows that using PL-Reg would exceed UCIR+LWS+L-Reg under most hyper-parameter settings. 
It can be noticed that different values of $w_{p1}$  may influence the model's performance.
Using a relatively smaller value of $w_{p1}$ such as $5e-4$ may lead to significantly decreased performance as it may affect the convergence of partial logic classification. 
Therefore, a relatively larger value of $w_{p1}$ would be recommended. 
The impact of $w_{p2}$ appears to be more nuanced and the performance are relatively stable. 
Based on the observed results, the combination of $w_{p1}= 1e-3$ and $w_{p1}= 1e-3$ offers the best overall performance.

\section{Conclusion and discussion}
\label{Discussion}
This paper argues the challenge of generalization in visual classification tasks involving unknown classes can be deeply connected to the scope of partial logic. Grounded in partial logic theory, 
the problems of GCD, mDG+GCD, and CIL tasks are reformed and the novel Partial Logic Regularization (PL-Reg) is proposed.
Unlike existing approaches such as L-Reg~\cite{tan2024interpret}, which are implicitly based on sentential logic and struggle with scenarios requiring undefined logic terms, especially in CIL tasks, PL-Reg allows models to maintain undefined logical terms during training by reserving capacity for unknown classes to be resolved in future stages. 
The efficacy of PL-Reg is validated in extensive experiments where significant improvements in unknown classes can be observed.

\textbf{Limitations.}
As one may have already noticed, the improvements brought by PL-Reg may be accompanied by compromises in the performance of known classes. The most severe case can be observed in  the TerraIncognita dataset in the mDG+GCD task. Such shortage is inherent in L-Reg, where both PL-Reg and L-Reg preassume the independence of each dimension of the semantic embeddings. 
Future work may focus on relaxing this presumption or developing backbones that can produce semantic embeddings to meet this presumption.

\textbf{Connections to gating mechanism.} This paper suggests an implicit connection between partial logic and the gating mechanism that has been broadly used in networks, such as the Long short-term memory architecture (LSTM)~\cite{hochreiter1997long} employing the forget gates. 
Those gates may help with building partial logic but may not ensure the construction of the sentential logic on the defined logical part. 
These connections show a promising approach for improving such gate-based methods and will be explored in the future. 





\textbf{Future work.}
Both PL-Reg and previous L-Reg are proposed under the discrete $Y$ tasks such as classification. Extending the theoretical results to those continuous $Y$ regression tasks may be possible.
Besides, the theoretical results and applications of the PL-Reg may be extended to language models, such as the discussed LSTM, or further.





\backmatter








\section*{Declarations}
All datasets used in this paper are publicly available.

\bibliographystyle{splncs04}
\bibliography{main.bib}

\end{document}